\documentclass[nohyperref]{article}

\usepackage{microtype}
\usepackage{graphicx}
\usepackage{booktabs} 

\usepackage{hyperref}

\usepackage[accepted]{icml2022}

\usepackage{amsmath}
\usepackage{amssymb}
\usepackage{mathtools}
\usepackage{amsthm}

\usepackage[capitalize,noabbrev]{cleveref}

\usepackage{amsmath}
\usepackage{amssymb}
\usepackage{amsthm}
\usepackage{adjustbox}
\usepackage{caption}
\usepackage{wrapfig}
\usepackage{subcaption}
\usepackage{booktabs}
\usepackage{xcolor}
\usepackage{soul}

\newtheorem{prop}{Proposition}

\theoremstyle{plain}

\theoremstyle{definition}

\theoremstyle{remark}

\usepackage[textsize=tiny]{todonotes}

\icmltitlerunning{How to Train Your Wide Neural Network Without Backprop}

\begin{document}

\twocolumn[
\icmltitle{How to Train Your Wide Neural Network Without Backprop: An Input-Weight Alignment Perspective}



\icmlsetsymbol{equal}{*}

\begin{icmlauthorlist}
\icmlauthor{Akhilan Boopathy}{MIT}
\icmlauthor{Ila Fiete}{MIT}
\end{icmlauthorlist}

\icmlaffiliation{MIT}{Massachusetts Institute of Technology}

\icmlcorrespondingauthor{Akhilan Boopathy}{akhilan@mit.edu}

\icmlkeywords{Machine Learning, ICML}

\vskip 0.3in
]



\printAffiliationsAndNotice{}  

\begin{abstract}
Recent works have examined theoretical and empirical properties of wide neural networks trained in the Neural Tangent Kernel (NTK) regime. Given that biological neural networks are much wider than their artificial counterparts, we consider NTK regime wide neural networks as a possible model of biological neural networks. Leveraging NTK theory, we show theoretically that gradient descent drives layerwise weight updates that are aligned with their input activity correlations weighted by error, and demonstrate empirically that the result also holds in finite-width wide networks. The alignment result allows us to formulate a family of biologically-motivated, backpropagation-free learning rules that are theoretically equivalent to backpropagation in infinite-width networks. We test these learning rules on benchmark problems in feedforward and recurrent neural networks and demonstrate, in wide networks, comparable performance to backpropagation. The proposed rules are particularly effective in low data regimes, which are common in biological learning settings.
\end{abstract}

\section{Introduction}
Deep neural networks trained with gradient descent are surprisingly successful in solving a variety of tasks and exhibit a moderate degree of  generalization~\cite{zhang2017understanding} and transferability~\cite{yosinski2014transferable}. However, it has remained theoretically difficult to understand what aspects of the training data the networks use -- and how -- to achieve their success. 
Recent theoretical work has established a number of convergence and generalization results for wide networks under the Neural Tangent Kernel (NTK) training regime because its analytical tractability  ~\cite{jacot2018NTK, arora2019exact, lee2019WideNN, lee2020finite}. Wide neural networks are also important models of \textit{biological neural networks}. Indeed, the evidence seems to suggest that neural circuits in the brain are substantially wider and shallower than current deep learning models~\cite{colonnier1981number}. Thus, wide (artificial) neural networks may more accurately model the brain. 

In this work, we leverage the tractability of the NTK regime of wide neural networks together with the suitability of wide neural networks as a model of biological systems to make progress in two directions: First, we investigate the learned representations of wide neural networks through the perspective of \textit{alignment} between the weights of a neural network and statistics of input data. Prior work has shown such alignment effects in the specialized settings of deep linear networks~\cite{saxe2019mathematical} and training with random labels~\cite{maennel2020what}; we extend these works by investigating alignment in nonlinear, wide neural networks. Our analysis indicates that even though neural network functions and training dynamics are complex and nonlinear, under an NTK training regime, intermediate layer weights of neural networks trained with gradient descent collect simple layerwise statistics of their inputs, weighted by errors on the output.

Second, the locality of the weights' dependence on the data propagated through the network prompts us to consider whether wide neural networks can be trained with simplified learning rules that reproduce the alignment effect without using backpropagation, the most common implementation of gradient descent in neural networks. Backpropagation is difficult to implement biologically,  and alternative learning rules may yield better understanding of how the brain learns~\cite{Seung03,Fiete06,bengio2016stdpvae, lillicrap2016random, liao2016important, nokland2016direct, bellec2020solution, bartunov2018assessing, richards2019deep, Roth19}. Biologically-plausible learning rules also can have practical advantages including easier implementation on neuromorphic hardware and reduced computational cost. In this work, we focus on data efficiency, which is practically relevant to many low data, real-world tasks. This advantage is also highly relevant to modeling learning in the brain; massive amounts of training data are often unavailable in biological learning settings.

In this paper, we make the following specific contributions\footnote{Our code is released at: \url{https://github.com/FieteLab/Wide-Network-Alignment}}:
\begin{itemize}
    \item We show theoretically that gradient descent on infinite-width neural networks in the NTK regime produces weights that are aligned with network layers in the following sense: the correlation matrix of the weight change at a particular layer is equal to an \textit{error-weighted} correlation matrix of the layer values at the same layer.
    \item We propose an \textit{alignment score} to quantify alignment in finite-width networks and empirically demonstrate input-weight alignment in wide, finite-width neural networks trained in the NTK regime.
    \item We develop a family of backpropagation-free learning rules that are designed to reproduce the input-weight alignment effect. Under the NTK training regime, we prove the equivalence between these learning rules and gradient descent in the infinite-width limit.
    \item We empirically show that the learning rules achieve comparable performance to gradient descent on wide convolutional neural networks (CNNs) and wide recurrent neural networks (RNNs) in the NTK training regime. The proposed learning rules outperform other backprop-free baseline learning rules including feedback alignment (FA)~\cite{lillicrap2016random} and direct feedback alignment (DFA)~\cite{nokland2016direct}. Moreover, the proposed Align learning rules are especially effective in low data settings, with Align-ada \textit{outperforming} gradient descent in very low data settings.
\end{itemize}

\section{Related Work}


\paragraph{Wide neural networks.} Recently, randomly initialized infinite-width deep neural networks have been shown to be equivalent to Gaussian processes~\cite{matthews2018gaussian, lee2018deep}. Moreover, gradient descent on infinite-width networks can be viewed as a kernel method: gradient descent on infinite width networks is equivalent to kernel regression under the Neural Tangent Kernel (NTK)~\cite{jacot2018NTK}. This insight has been used to train neural networks exactly in the infinite limit~\cite{arora2019exact}. NTK theory has also been used to show theoretically and empirically that wide \textit{finite-width} neural networks evolve as \textit{linearized} models in terms of their parameters~\cite{lee2019WideNN}. \cite{lee2020finite} further shows that the correspondence between finite-width wide networks and infinite width networks is strongest at small learning rates among other results. Networks trained in the NTK regime are also practically effective on low data tasks~\cite{arora2020harnessing}.

We highlight that many results in the above works are only applicable under a specific \textit{neural tangent} parameterization, or equivalently under standard parameterization under a specific choice of width-dependent learning rate~\cite{lee2019WideNN}, which leads to the NTK infinite-width limit. Recent work has shown that there are an infinite number of infinite-width limits, with certain limits, including the NTK limit, exhibiting an equivalence to kernel regression and other limits allowing for \textit{feature learning} in the infinite-limit~\cite{yang2020feature}. In this work we focus on the NTK limit as it is the most studied and has the most analytical tools available. We leave an extension of our results to other infinite-width limits as future work to be explored.
\paragraph{Biologically-plausible learning.}
In an effort to develop accurate models of learning in the brain, researchers have developed a number of biologically-plausible learning algorithms for neural networks. These algorithms avoid some of the biologically implausible aspects of backpropagation, the standard algorithm to train artificial neural networks. These aspects include 1) the fact that forward propagation weights must equal backward propagation weights throughout the training process, known as the weight transport problem~\cite{lillicrap2016random}, 2) the asymmetry of a nonlinear forward propagation computation and a linear backward propagation computation~\cite{bengio2016stdpvae} and 3) the alternation of forward and backward passes through the network and the computation of nonlinear activation function derivatives during the backward pass~\cite{nokland2016direct}.

Sign symmetry~\cite{liao2016important, xiao2018biologicallyplausible} and feedback alignment (FA)~\cite{lillicrap2016random, song2021convergence} address weight transport by showing that backward weights need not equal the corresponding forward weights for effective learning. Specifically, sign symmetry imposes \textit{sign-concordance} between the forward and backward weights, relaxing the need for continual transport of the exact forward weights to the backward weights. Removing the need for a biological mechanism to continuously synchronize two different synaptic weights is arguably a major step towards biological plausibility. As we will show, our proposed methods, similar to sign symmetry, avoid continual weight transport
but require a correspondence between forward and backward weights at initialization. Building on FA, Direct feedback alignment (DFA) further shows the effectiveness of passing feedback directly to intermediate layers, alleviating the need for a sequential layerwise backward pass~\cite{nokland2016direct,launay2020direct, refinetti2021align}. Greedy layerwise learning takes a different approach and trains different layers of a neural networks independently, avoiding global feedback altogether~\cite{nokland2019training}. Many other learning rules \textit{learn} a feedback pathway, including target propagation-based algorithms~\cite{lee2015difference,meulemans2020theoretical,ororbia2018conducting,ororbia2019biologically,ororbia2020continual,ororbia2020largescale,manchev2020target}, weight mirroring~\cite{akrout2019deep}, predictive coding~\cite{millidge2020predictive,ororbia2022neural}, and eligibility trace-based algorithms for RNNs~\cite{Roth19,marschall2020unified,williams1989learning,bellec2020solution}. Some of these biologically-motivated algorithms can have a number of practical advantages relative to backpropagation including 1) asynchronous updating of weights at different layers of a network, 2) reduced memory costs from having to store intermediate layer activation values, 3) reduced synaptic wiring in the feedback path. The resulting computational efficiencies can be particularly great on neuromorphic hardware, where forward and backward network weights are represented by physically separate wiring on a VLSI circuit.

Currently proposed biologically-motivated algorithms often require tuning feedback weights to be aligned with forward weights to achieve good performance; algorithms entirely avoiding such weight transport have difficulties scaling to larger-scale tasks~\cite{bartunov2018assessing}. In this paper, we propose learning rules that avoid any tuning of feedback weights. However, unlike previous works we prove an equivalence between our learning rules and backpropagation in the limit of infinite-width networks. We believe this provides a strong theoretical reason to expect our learning rules to be as scalable as gradient descent given sufficiently wide networks.
\section{Alignment Between Layers and Weights}
In this section, we introduce the concept of alignment between the layers and weights of a neural network. We then prove that this alignment occurs in infinite-width networks, showing that weights in infinite width networks capture layerwise statistics of input data. Furthermore, we introduce an \textit{alignment score} metric to quantify alignment in finite width networks. Using this metric, we empirically demonstrate alignment between layers and weights in finite-width networks.  
\subsection{Input-Weight Alignment} We consider an $N$-layer neural network $f(x)$ with input $x$. We denote the activation function as $\sigma$ and weights and biases $W_l$ and $b_l$ at layer $l$. Throughout this paper, we use the neural tangent parameterization of weights~\cite{jacot2018NTK}. We denote the pre-activation layer values at layer $l$ as functions $z_l(x)$ of input $x$. The layer pre-activations are defined as:
\begin{equation} \label{eqn: ntk_param}
    z_{l}(x)  = \frac{1}{\sqrt{m_{l-1}}} W_l \sigma(z_{l-1}(x)) + b_l 
\end{equation}
for $l = 1, ... N$ where $f(x) = z_N(x)$ and $m_{l-1}$ is the dimensionality of $z_{l-1}(x)$. For notational convenience we denote $\sigma(z_{0}(x)) = x$. Next, following the settings of~\cite{jacot2018NTK, arora2019exact, lee2019WideNN}, we assume that the network parameters are trained with continuous time gradient flow of learning rate $\eta$ on a supervised learning task with loss function $\mathbb{E}_{p_x}[L(f(x), y(x))]$ where $y(x)$ are targets and $p_x$ denotes the finite-supported training distribution over inputs $x$. We denote the parameter values at time $t$ of training with superscript $^{(t)}$: layer $l$ weights and biases at training time $t$ are denoted $W_l^{(t)}$ and $b_l^{(t)}$. Similarly, we denote the overall network function and intermediate-layer pre-activations at time $t$ as $f^{(t)}(x)$ and $z_l^{(t)}(x)$ respectively.

To define the concept of alignment, it is convenient to define a layerwise \textit{weight-change correlation matrix} $\Delta^{(t)}_l$:
\begin{equation}
    \Delta^{(t)}_l = (W_l^{(t)}-W_l^{(0)})^T (W_l^{(t)}-W_l^{(0)}).
\end{equation}
This matrix represents the correlation between different \textit{columns} of the weight change since initialization, $W_l^{(t)}-W_l^{(0)}$. This matrix captures all the information about $W_l^{(t)}-W_l^{(0)}$ up to rotations of the columns since if $W_l^{(t)}-W_l^{(0)}$ has singular value decomposition $U S V^T$, then $\Delta^{(t)}_l = V S^T S V^T$. Intuitively, the elements of $\Delta^{(t)}_l$ represent the similarity between post-activation units of layer $l-1$ in terms of how they impact the next layer. We focus on weight change correlation instead of weight correlation $W_l^{(t)T} W_l^{(t)}$ because in the NTK training regime, the scale of weight initializations is larger than weight updates~\cite{jacot2018NTK}. Thus, the weight correlation 
 does not significantly change over the course of training, and it is more relevant to consider the weight change correlation.

We also define a layerwise \textit{weighted input activity correlation matrix} $\Sigma^{(t)}_{l,q}$ as the correlation between post-activations under a particular pairwise weighting of inputs $q(x_1, x_2)$:
\begin{equation}
    \Sigma^{(t)}_{l,q} \\ = \mathbb{E}_{x_1\sim p_x, x_2\sim p_x}[ \sigma(z_{l-1}^{(t)}(x_1)) q(x_1, x_2) \sigma(z_{l-1}^{(t)}(x_2))^T]
\end{equation}
Finally, we say that activities and weights are aligned for a weighting $q$ when the activity correlation matrix is proportional to the weight change correlation matrix:
\begin{equation}
    \Sigma^{(t)}_{l,q} = k \Delta^{(t)}_l
\end{equation}
for some positive constant $k$. When the weights are aligned with input activities, the weights can be fully specified as a function of the corresponding layer's input correlation matrix up to rotation. In other words, alignment implies that weights capture layer-local statistics of the input data.

Note that our notion of alignment is stronger than the one proposed in~\cite{maennel2020what} since we not only require the eigenvectors of the weight change correlation matrix and input activity correlation matrix to be equal, but also require the corresponding eigenvalues to be proportional to each other. We also note the following additional differences with~\cite{maennel2020what}: 1) we use the weight change correlation matrix instead of the weight correlation matrix, 2) we consider the input correlation matrix instead of the input \textit{covariance} matrix which normalizes layer activations to have mean zero, 3) we compute input correlation matrices over a general pairwise weighting $q(x_1, x_2)$ instead of restricting the weighting to the delta function $q(x_1, x_2) = \delta(x_1, x_2)$.
\subsection{Alignment in infinite-width}
In the infinite-width limit, layerwise inputs and weights are aligned under a specific weighting $q_l^{(t)}(x_1, x_2)$ that depends only on the gradients of the network at initialization and an \textit{integrated-error} function $\delta^{(t)}(x) \in \mathbb{R}^{m_N}$: 
\begin{equation}
    \delta^{(t)}(x) = \eta \int_{0}^t \nabla_{z_N} L(f^{(\tau)}(x), y(x)) d \tau
\end{equation}
This quantity captures the averaged influence $x$ has on the function $f(x)$ over the course of training. Note that this quantity only depends on the output values of the network and effectively measures the average error incurred by the network outputs on point $x$ across the course of training. We also define a layerwise kernel $\Gamma_l(x_1, x_2) \in \mathbb{R}^{m_N \times m_N}$ as:
\begin{equation}
    \Gamma_l(x_1, x_2) = \nabla_{z_l} f^{(0)}(x_1)^T \nabla_{z_l} f^{(0)}(x_2)
\end{equation}
This quantifies the similarity in the gradients of $x_1$ and $x_2$ at layer $l$. It depends only on the initial configuration of the network. Finally, we define $q^{(t)}_l(x_1, x_2) \in \mathbb{R}$ as:
\begin{equation}
    q^{(t)}_l(x_1, x_2) = \delta^{(t)}(x_1)^T \Gamma_l(x_1, x_2) \delta^{(t)}(x_2)
\end{equation}
\begin{figure*}[htbp]
    \centering
    \begin{subfigure}{0.48\textwidth}
      \centering
      \includegraphics[width=\linewidth]{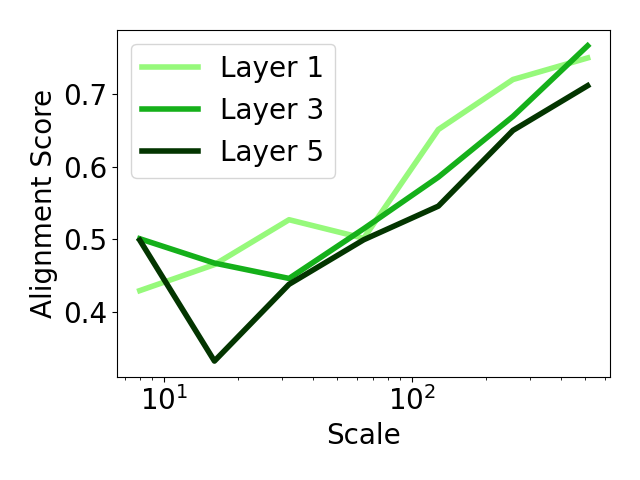}
      \caption{CIFAR-10}
      \label{fig:align_width_cifar}
    \end{subfigure}%
    \begin{subfigure}{0.48\textwidth}
      \centering
      \includegraphics[width=\linewidth]{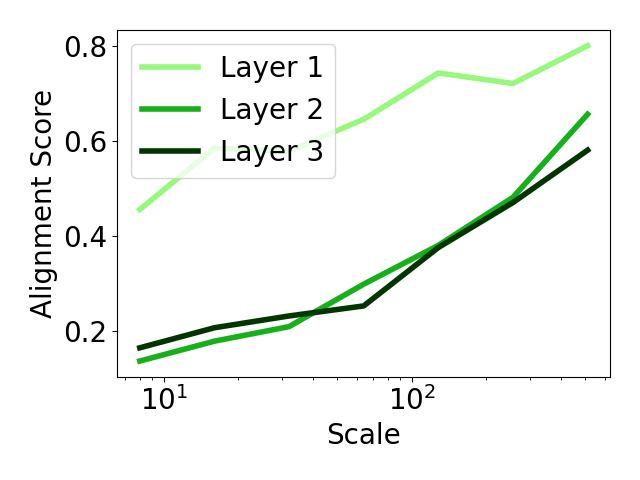}
      \caption{KMNIST}
      \label{fig:align_width_kmnist}
    \end{subfigure}
    \caption{Alignment scores at different layers of convolutional neural networks with different layer widths trained on the CIFAR-10 (8 layers) and KMNIST (4 layers). The scale of the architecture is denoted $\times n$, where $n$ is the number of filters in intermediate layers.}
    \label{fig:align_width}
\end{figure*}
Intuitively, pairs of points which influence the network $f(x)$ similarly will be weighted positively and points which have opposite influences on the network will be weighted negatively. Note that this weighting produces an input activity correlation $\Sigma^{(t)}_{l,q^{(t)}_l}$ whose only dependence on the targets $y(x)$ is through the integrated-error $\delta^{(t)}(x)$. Therefore, $\Sigma^{(t)}_{l,q^{(t)}_l}$ is an \textit{error-weighted} correlation matrix of the data. We can then show alignment between input activations and weights in the infinite-width limit: 
\begin{prop} \label{prop: inf_width_alignment}
Assume layers are randomly initialized according to the neural tangent initialization~\cite{jacot2018NTK}. Also, assume that the nonlinearity $\sigma(\cdot)$ has bounded first and second derivatives. Suppose the network is trained with continuous gradient flow with learning rate $\eta$ on the mean-squared error loss for $T$ time such that $\max_x ||y(x)-f^{(t)}(x)||_2$ is uniformly bounded in $t \in [0, T]$. Then, in the simultaneous limit $\lim_{m_{N-1}, m_{N-2}, ... m_1 \to \infty}$:
    \begin{equation}
        \frac{1}{\sqrt{m_{l-1}}} || \frac{1}{m_{l-1}} \Sigma^{(0)}_{l,q^{(t)}_l} -  \Delta^{(t)}_l ||_{op} \in \mathcal{O}(\frac{1}{\min \{ m_1, ... m_{N-1} \} })
    \end{equation}
\end{prop}
\begin{proof}
    See Appendix~\ref{prop1_proof} for a proof.
\end{proof}
Proposition~\ref{prop: inf_width_alignment} equates the layerwise input correlation at time $0$ with $m_{l-1} \Delta_l^{(t)}$. Conceptually, as layer widths increase, the feedback Jacobian $\nabla_{z_l} f^{(t)}(x)$ and intermediate-layer activations change vanishingly little over the course of training, making the weight change correlation proportional to the input correlation. See Appendix~\ref{prop1_proof} for a discussion of the width dependent factors in the above expression. The interpretation of the result hinges on the nature of the pairwise weighting $q_l^{(t)}(x_1, x_2)$. As discussed, $q_l^{(t)}(x_1, x_2)$ intuitively captures the similarity in the errors from $x_1$ and $x_2$. Thus, Proposition~\ref{prop: inf_width_alignment} can be interpreted as saying that the weight changes of an infinite-width neural network reflect error-weighted correlation statistics of the inputs from corresponding intermediate layers.

\subsection{Alignment in finite-width}
Although Proposition~\ref{prop: inf_width_alignment} applies to infinite-width networks, we show that alignment approximately holds in wide finite-width networks. We empirically quantify alignment through an \textit{alignment score} that measures the similarity between the weight change correlation $\Delta^{(t)}_l$ and the time $0$ input correlation $\Sigma^{(0)}_{l,q^{(t)}_l}$. We wish
to use a matrix similarity measure that is scale invariant, while being maximized for proportional matrices. Cosine similarity is a natural and commonly used measure with these properties, which we use to define an alignment score $S$:
\begin{equation}
    S = tr\bigl(\Delta^{(t)}_l \Sigma^{(0)}_{l,q^{(t)}_l}\bigr)
    tr\bigl(\Delta^{(t)2}_l\bigr)^{-1/2} tr\bigl(\Sigma^{(0)2}_{l,q^{(t)}_l}\bigr)^{-1/2}
\end{equation}
This score is in the range $[-1, 1]$ with $1$ indicating perfect alignment. For wide neural networks, it may not be computationally efficient to operate directly on $\Delta^{(t)}_l$ and $\Sigma^{(0)}_{l,q^{(t)}_l}$ because their sizes scale quadratically with layer width. Instead, the score can be computed in terms of $\Delta^{(t)}_l z$ and $\Sigma^{(0)}_{l,q^{(t)}_l} z$ when $z$ is sampled from a unit Gaussian $N(0, I)$:
\begin{equation} \label{eqn: align_score}
    S = \mathbb{E}_{z}\bigl[z^T\Delta^{(t)}_l \Sigma^{(0)}_{l,q^{(t)}_l} z]{\biggl(\mathbb{E}_{z}\bigl[||\Delta^{(t)}_l z||_2^2\bigr] \mathbb{E}_{z}\bigl[ ||\Sigma^{(0)}_{l,q^{(t)}_l} z||_2^2\bigr]\biggr)^{-1/2}}
\end{equation}
$\Delta^{(t)}_l z$ and $\Sigma^{(0)}_{l,q^{(t)}_l} z$ are computed as a series of matrix-vector products:
\begin{equation}
    \Delta^{(t)}_l z = (W_l^{(t)}-W_l^{(0)})^T (W_l^{(t)}-W_l^{(0)}) z
\end{equation}
\begin{equation}
    \Sigma^{(0)}_{l,q^{(t)}_l} z = m_{l-1} (W_l^{*(t)}-W_l^{(0)})^T (W_l^{*(t)}-W_l^{(0)}) z
\end{equation}
where $W_l^{*(t)}$ is defined by the following dynamics, starting from initial condition $W_l^{*(0)} = {W}_l^{(0)}$:
\begin{multline} \label{lin_align_rule}
    \dot{W}_l^{*(t)} =  \frac{\eta}{\sqrt{m_{l-1}}} \times \\  \mathbb{E}_{p_x}\bigl[\nabla_{z_l} f^{(0)}(x) \nabla_{z_N} L(f^{(t)}(x), y(x)) \sigma(z_{l-1}^{(0)}(x))^T \bigr].
\end{multline}
${W}_l^{*(t)}$ is computed by simply training a network with the above learning rule. This alignment computation naturally yields an interpretation of the alignment score as a similarity between the parameters of a network and the corresponding network trained with initialized activations $\sigma(z_{l-1}^{(0)}(x))$ and initialized feedback Jacobian $\nabla_{z_l} f^{(0)}(x)$.  We further investigate variants of the learning rule of Equation~\eqref{lin_align_rule} in  Section~\ref{learning_rules}.

Next, we use the alignment scores to evaluate alignment in finite-width networks trained on CIFAR-10~\cite{krizhevsky2009cifar} and KMNIST~\cite{clanuwat2018kmnist}, as a function of network width. We train an 8 (CIFAR-10) or 4 (KMNIST) layer ReLU-activated CNN on the mean squared error loss and vary the number of convolutional filters in the intermediate layers from $8$ to $512$. See Section~\ref{experiments} for further architectural and hyperparameter details. In Figure~\ref{fig:align_width}, we find that alignment scores increase with network width, consistent with our prediction of perfect alignment in infinite width. Generally, alignments are lower in later layers of network, suggesting the infinite-width correspondence holds best in early layers. Nevertheless, the results illustrate that in sufficiently wide finite-width networks, intermediate layer weights indeed capture local statistics of the corresponding layers' inputs. See Appendix~\ref{app: tables} Tables~\ref{tab: alignment_metrics_cifar},~\ref{tab: alignment_metrics_kmnist} for full tabulated results for these networks. In Appendix~\ref{app: tables} Table~\ref{tab: alignment_metrics_cifar_tanh}, we also tabulate alignment scores for an 8 layer Tanh-activated CNN trained on CIFAR-10; we find qualitatively similar patterns indicating that the alignment effect empirically holds across activation functions. In Appendix~\ref{app: tables} Figure~\ref{fig:align_epoch}, we plot the evolution of alignment scores during training find that they are initially high and slowly decrease during training as parameters drift from initialization. We also plot baseline alignment scores in networks with randomly permuted weights after training and find low alignment, validating the statistical significance of our alignment scores.

\subsection{Discussion}
The alignment results demonstrate that the weights of a wide neural network trained by gradient descent can be described by \textit{local} layerwise activation statistics. This is surprising since in general settings, intermediate layer weights can encode \textit{global} information about a network; we show that in the wide network regime, this role of global information is limited. The alignment results directly link trained weights to activation statistics at their corresponding layers, providing greater interpretability into trained networks. To our knowledge, this result is the most general alignment result of its kind found so far. Moreover, as we will demonstrate next, it both conceptually and mathematically motivates simplified learning rules which rely less on coordinated global feedback than backpropagation.

\section{Simplified Training Rules Produce Alignment} \label{learning_rules}
Prompted by the observation that wide finite-width networks exhibit high input-weight alignment, we propose a family of simplified learning rules designed to produce the alignment effect. We show theoretically that the learning rules are equivalent to gradient descent in the infinite width limit. In Section~\ref{experiments}, we perform extensive experiments assessing the performance of these learning rules in finite-width networks.
\begin{figure}
    \centering
    \includegraphics[scale=0.49]{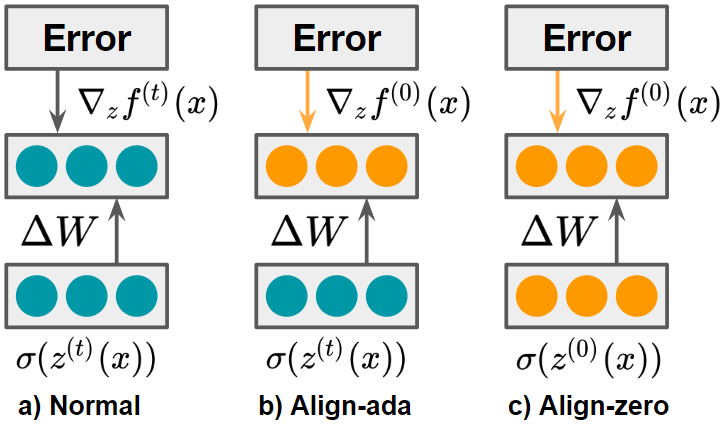}
    \caption{Visual comparison of weight updates $\Delta W$ in normal training (via gradient descent) and alignment-based training (Align-ada and Align-zero).}
    \label{fig:align_fig}
\end{figure}
\begin{figure*}
    \centering
    \begin{subfigure}{.49\textwidth}
      \centering
      \includegraphics[width=\linewidth]{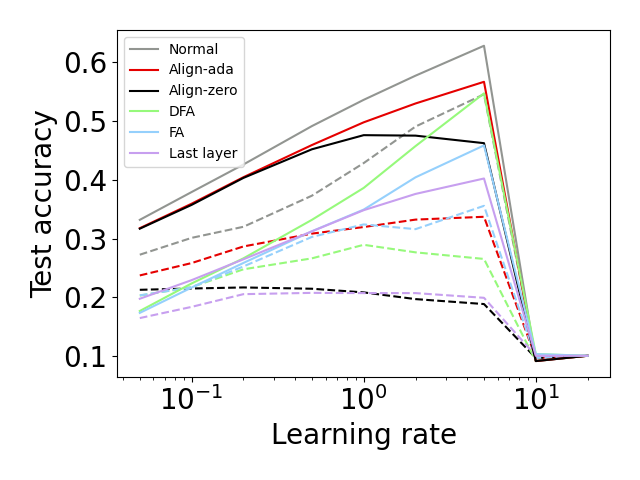}
      \caption{CIFAR-10}
      \label{fig:lr_tuning_cifar}
    \end{subfigure}%
    \begin{subfigure}{.49\textwidth}
      \centering
      \includegraphics[width=\linewidth]{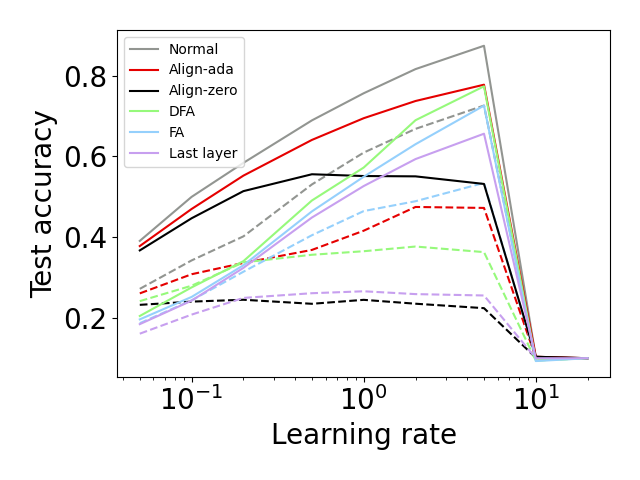}
      \caption{KMNIST}
      \label{fig:lr_tuning_kmnist}
    \end{subfigure}
    \caption{Test set accuracies of different learning rules vs. learning rate on CNNs trained on CIFAR-10 and KMNIST. Solid lines: $256$ conv. filters per layer; dashed: $8$ conv. filters.}
    \label{fig:lr_tuning}
\end{figure*}
Before considering simplified training rules, we examine the learning rule of gradient descent:
\begin{multline} \label{normal_rule}
    \dot{W}_l^{(t)} =  \frac{\eta}{\sqrt{m_{l-1}}} \times \\ \mathbb{E}_{p_x}[\nabla_{z_l} f^{(t)}(x) \nabla_{z_N} L\bigl(f^{(t)}(x), y(x)\bigr) \sigma(z_{l-1}^{(t)}(x))^T ]
\end{multline}
For concision, we omit the learning rule for biases which is found by removing the width scaling factor and $\sigma(z_{l-1}^{(t)}(x))^T$. In the infinite-width limit, we expect $\nabla_{z_l} f^{(t)}(x)$ and $\sigma(z_{l-1}^{(t)}(x))$ to remain unchanged over the course of training, which enables input-weight alignment. The simplified learning rules are found by setting both of these quantities to their values at initialization. Specifically, we find the \textit{Align-zero} learning rule by setting both quantities to their initialization values at time $0$:
\begin{multline} \label{lin_rule_full}
    \dot{W}_{l, al\cdot 0}^{(t)} =   \frac{\eta}{\sqrt{m_{l-1}}} \times \\ \mathbb{E}_{p_x}[\nabla_{z_l} f^{(0)}(x) \nabla_{z_N} L(f_{al\cdot 0}^{(t)}(x), y(x)) \sigma(z_{l-1}^{(0)}(x))^T]
\end{multline}
where $f_{al\cdot 0}^{(t)}$ represents the network with parameter values $W^{(t)}_{l,al \cdot 0}$. The difference between the rule of Equation~\eqref{lin_rule_full} and the earlier alignment computation rule of Equation~\eqref{lin_align_rule} is that the previous rule uses output gradients $\nabla_{z_N} L(f^{(t)}(x), y(x))$ of the original network $f(x)$ trained with gradient descent, while the above rule uses output gradients $\nabla_{z_N} L(f_{al\cdot 0}^{(t)}(x), y(x))$ of the network $f_{al\cdot 0}(x)$ trained with the above learning rule. Unlike the earlier rule, the above learning rule does not require an additional network $f(x)$ trained with gradient descent; it can be used by itself to compute $f_{al\cdot 0}^{(t)}(x)$. Note that Align-zero appears similar to the learning rule of~\cite{lee2019WideNN}; see Appendix~\ref{remark} for a discussion of the key differences.



Recognizing that network activations might change slightly in wide but finite-width networks, we may alter the Align-zero rule to allow parameter learning to \textit{adapt} to the locally available values  $\sigma(z_{l-1}^{(t)}(x))$, instead of fixing them to their values at initialization. This results in the \textit{Align-ada} learning rule:
\begin{multline} \label{slin_rule_full}
    \dot{W}_{l, al\cdot ada}^{(t)} =   \frac{\eta}{\sqrt{m_{l-1}}} \times \\ \mathbb{E}_{p_x}[\nabla_{z_l} f^{(0)}(x) \nabla_{z_N} L(f_{al\cdot ada}^{(t)}(x), y(x)) \sigma(z_{l-1}^{(t)}(x))^T ]
\end{multline}
where $f_{al\cdot ada}(x)$ denotes the network with parameters ${W}_{l, al\cdot ada}^{(t)}$ trained by Align-ada. See Appendix~\ref{app: pseudocode}, Algorithms~\ref{algo: linearized} and~\ref{algo: semilinearized}, for pseudocode and full procedures for training with Align-zero and Align-ada; the difference between the two algorithms is highlighted. Also see Figure~\ref{fig:align_fig} for a visual comparison of Align-zero and Align-ada.

Both Align-zero and Align-ada avoid continual weight transport as they replace the backpropagation Jacobian $\nabla_{z_l} f^{(t)}(x)$ with a fixed feedback Jacobian $\nabla_{z_l} f^{(0)}(x)$. However, the rules still require computing the fixed-over-learning feedback function $\nabla_{z_l} f^{(0)}(x)$, the backpropagation Jacobian at time $0$. This may difficult to compute biologically; nevertheless, we believe that such a fixed feedback function would be significantly easier to explain biologically compared to backpropagation which requires constant synchronization or mirroring of the forward and feedback weights. Indeed, previously proposed biologically-motivated learning rules also avoid weight transport, but impose some similarity between forward and backward weights~\cite{liao2016important}.

Next, we show that Align-ada and Align-zero are \textit{equivalent} to gradient descent in the limit of infinite-width:
\begin{prop} \label{prop: inf_width_linearized}
    Suppose networks of the same initialization are trained with Align-ada, Align-zero and gradient descent under the setting of Proposition 1. Then, in the simultaneous limit $\lim_{m_{N-1}, m_{N-2}, ... m_1 \to \infty}$:
    \begin{multline}
        ||f^{(t)}(x) - f_{al\cdot 0}^{(t)}(x)||,\,\,\, ||f^{(t)}(x) - f_{al\cdot ada}^{(t)}(x)|| \\ \in \mathcal{O}(\frac{1}{\sqrt{\min \{ m_1, ... m_{N-1} \} }})
    \end{multline}
\end{prop}
\begin{proof}
    See Appendix~\ref{prop2_proof} for a proof.
\end{proof}
This result establishes that simplified, backprop-free learning rules designed to produce input-weight alignment are \textit{as effective as gradient descent in the infinite-width limit}. Moreover, given our observation that wide finite-width networks exhibit high input-weight alignment, Proposition~\ref{prop: inf_width_linearized} suggests that these learning rules might be effective loss-optimizing rules in finite-width networks. Finally, while we use the notation of feedforward networks, our analysis also applies to recurrent neural networks and our learning rules can be adapted to recurrent networks as shown in Appendix~\ref{app: rnn_impl}.

Although Proposition 2 shows that the specific methods Align-ada and Align-zero are equivalent to gradient descent in the infinite width limit, we believe our analysis is applicable to a large family of learning rules. As a specific example, in Appendix~\ref{app:alignprop}, we propose an additional Align variant named \textit{Align-prop} in which the approximation of the backpropagation Jacobian $\nabla_{z_l} f^{(t)}(x)$ is allowed to change over the course of training. Like Align-ada and Align-zero, Align-prop is equivalent to gradient descent in infinite width networks. The presence of multiple rules equivalent to gradient descent in infinite width networks suggests the possibility of many potentially biologically-plausible learning rules with strong theoretical guarantees in wide neural networks.

In the next section, we investigate the empirical effectiveness of Align methods.

\section{Experiments on Task Performance} \label{experiments}
In this section, we train finite-width networks using Align-ada and Align-zero and compare their performance to normal training by gradient descent and other backprop-free baseline learning rules. We assess the efficacy of Align-ada and Align-zero across different network widths and learning rates. See Appendix~\ref{app:alignprop} for experiments on Align-prop. We conduct experiments on CNNs trained on CIFAR-10~\cite{krizhevsky2009cifar}, Small CIFAR-10~\cite{arora2020harnessing}, KMNIST~\cite{clanuwat2018kmnist} and ImageNet~\cite{russakovsky2015Imagenet} and RNNs trained on an addition task. 
Note that our goal with these experiments is not to optimize performance to beat state-of-the-art results on the tasks but rather to demonstrate both the validity of the theoretical work for finite-width networks and to show that Align methods approach the performance of backpropagation while outperforming biologically-motivated baselines. Although we primarily consider the NTK training regime, in Appendix~\ref{app: std_param} we present results under standard parameterization. In Appendix~\ref{app:multiseed}, we conduct additional experiments over multiple random seeds to assess the statistical significance of our results. In Appendix~\ref{app:train_time}, we evaluate the training time of Align methods vs. backpropagation.
\begin{figure*}
    \centering
    \begin{subfigure}{.49\textwidth}
      \centering
      \includegraphics[width=\linewidth]{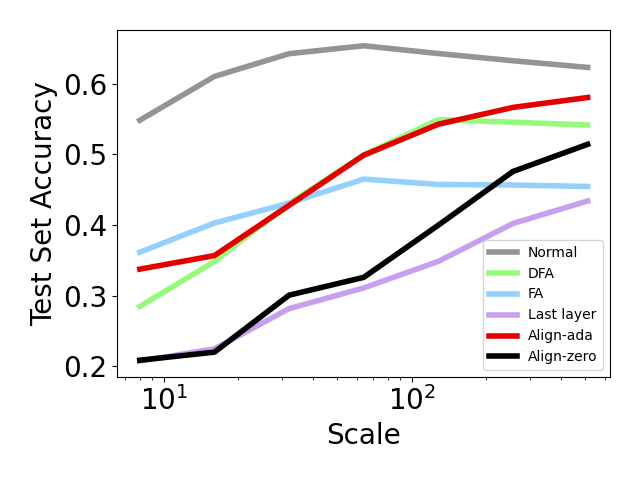}
      \caption{CIFAR-10}
      \label{fig:width_tuning_cifar}
    \end{subfigure}%
    \begin{subfigure}{.49\textwidth}
      \centering
      \includegraphics[width=\linewidth]{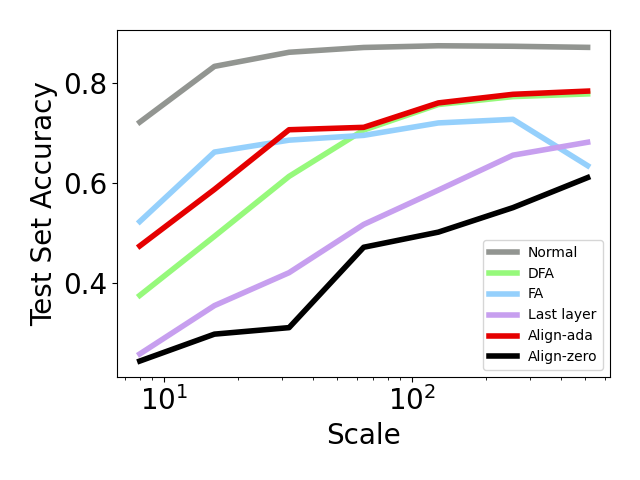}
      \caption{KMNIST}
      \label{fig:width_tuning_kmnist}
    \end{subfigure}
    \caption{Test set accuracies of different learning rules vs. network scale of CNNs trained on CIFAR-10 and KMNIST.}
    \label{fig:width_tuning}
\end{figure*}
\subsection{Results on Add task in RNNs}
We defer results on the Add task to Appendix~\ref{app: add_task}; in summary, we find that Align methods perform comparably to gradient descent in wide recurrent neural networks.

\subsection{Results on CIFAR-10 and KMNIST}
\paragraph{Dataset and baselines}
We compare various learning methods (gradient descent (denoted Normal), training only the last layer of the network (denoted Last layer), DFA~\cite{nokland2016direct}, FA~\cite{lillicrap2016random}) on CNNs with 7 (CIFAR-10) or 3 (KMNIST) convolutional layers followed by global average pooling and a final fully connected layer. We defer additional architecture and hyperparameter details and descriptions of baselines to Appendix~\ref{app: hyperparams}.

\paragraph{Varying learning rate} \label{exp: vary_lr}
In Figure~\ref{fig:lr_tuning}, we plot the performance of Align methods vs. baselines over a grid search of learning rates in $[0.05, 20]$ for a narrow and wide network. At small learning rates, the performance of Align-ada, Align-zero and Normal is very close on both CIFAR-10 and KMNIST. This may be because of higher input-weight alignment at lower learning rates, and therefore greater correspondence between the Align learning rules and backprop. The gap between the methods grows modestly with increased learning rates and, as theoretically expected, the gap grows faster in the narrow network. On the wide network, across the entire range of learning rates tested, Align-ada outperforms all non-backprop baselines. All methods have a sharp drop in performance at the same learning rate value of $10$, at which point training becomes unstable.
\begin{figure}
    \centering
    \includegraphics[width=0.99\linewidth]{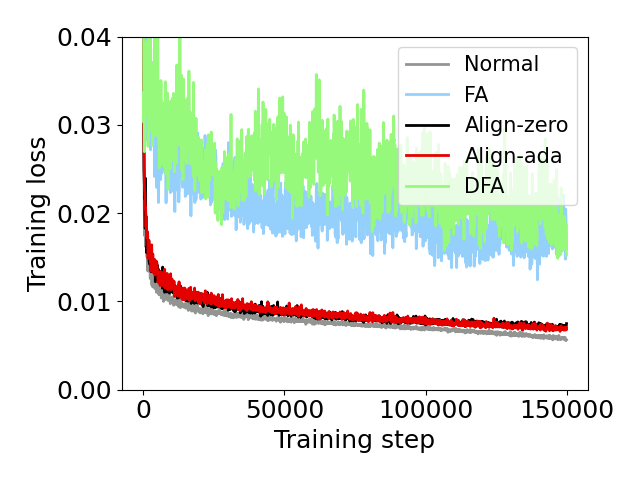}
    \caption{Training loss curves of different learning rules.}
    \label{fig:imagenet}
\end{figure}
\paragraph{Varying network width}
In Figure~\ref{fig:width_tuning}, we compare performance of different learning rules under architectures of different widths. While Normal performs well at all network widths, Align-ada performs nearly as well as gradient descent at large widths; Align-zero performs worse, but still significantly improves with network width. Align-zero's relatively worse performance suggests that layers $\sigma(z^{(t)}_{l-1}(x))$ vary considerably more over the course of training than gradients $\nabla_{z_l} f^{(t)}(x)$ in finite width networks. At the largest width tested, Align-ada matches or outperforms all baselines except Normal on both CIFAR-10 and KMNIST. See Appendix~\ref{app: tables} for full tabulated results.

\paragraph{Note on performance}
While our accuracies are generally low relative to standardly trained networks, they are in similar ranges to prior results for comparably sized networks trained with NTK-based linearized training methods~\cite{lee2019WideNN}. This validates that our wide networks indeed are in the NTK regime. It may be possible to improve performance by extending Align methods to non-NTK regimes via, for example, slow timescale weight transport of backward weights; we leave such extensions as a future work.
\begin{figure}
    \centering
      \includegraphics[width=0.99\linewidth]{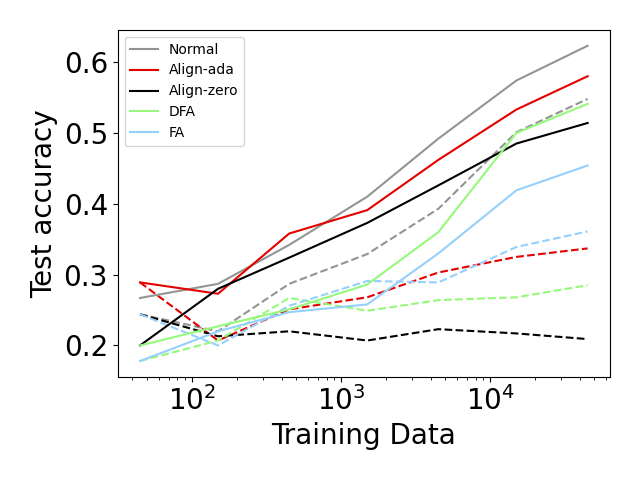}
    \caption{Test set accuracies of different learning rules vs. training set size on CNNs trained on Small CIFAR-10. Solid lines: 512 conv. filters per layer, dashed: 8 conv. filters.}
    \label{fig:data_tuning}
\end{figure}
\subsection{Results on ImageNet}
\paragraph{Dataset and baselines}
We train an 8 layer CNN on ImageNet, comparing Align-ada and Align-zero to Normal, FA and DFA. Additional details on the dataset, architecture, hyperparameters and baselines are included in Appendix~\ref{app: hyperparams}. We aim not to achieve competitive performance on ImageNet, but to illustrate that Align methods can closely match the training dynamics of normal training in the NTK regime even on challenging real-world datasets.
\paragraph{Performance during training}
In Figure~\ref{fig:imagenet}, we find that the training losses of Align-ada and Align-zero closely match that of Normal, with a small gap only emerging near $1.5\times 10^{5}$ training steps. By contrast, DFA and FA do not approach the performance of Normal. This suggests that Align methods may be significantly more scalable than other backprop-alternative learning rules.

\subsection{Results in the low-data regime: Small CIFAR-10}
\paragraph{Dataset and baselines}
We construct Small CIFAR-10 as a random subset of points from the full CIFAR-10 training set. Otherwise, we use the same settings as CIFAR-10.

\paragraph{Varying amount of training data}
In Figure~\ref{fig:data_tuning}, we plot the performance of Align methods vs. baselines when trained on different amounts of CIFAR-10 training data. At low levels of training data, Align methods perform comparably to Normal, with Align-ada consistently outperforming Normal on the lowest data setting. Low data settings may induce high input-weight alignment, and thereby similar learning dynamics between Align rules and backprop; this can allow better relative performance of Align rules. Align rules may also have a regularizing effect, allowing them to outperform backprop in very low data settings. Note that in the low data regime, Align methods perform relatively well even in the narrow network. The results demonstrate the efficacy of Align methods in low data regimes, which may be common in biological learning settings. See Appendix~\ref{app: tables} for full tabulated results.

\section{Conclusion}
In this paper, we found that weights of infinite width networks and wide finite width networks trained by backprop capture simple layerwise statistics of their inputs weighted by output errors. We used this finding to develop backprop-free learning rules that are equivalent to gradient descent in infinite-width networks. In a variety of empirical settings, the rules are comparable to gradient descent and outperform biologically-motivated baselines, and are especially effective on low data tasks.

\bibliography{example_paper}
\bibliographystyle{icml2022}

\newpage
\renewcommand{\thesubsection}{\Alph{subsection}}

\appendix
\onecolumn
\section*{Appendix}

\subsection{Proof of Proposition~\ref{prop: inf_width_alignment}} \label{prop1_proof}
We first prove a lemma that will be used in the proofs of both Propositions 1 and 2. The lemma intuitively demonstrates that 1) the weights and activations of infinitely wide neural networks change an infinitesimally small amount during training, 2) thus, the feedback Jacobian $\nabla_\theta f^{(t)}(x)$ changes infinitesimally little during training, 3) the approximations of the feedback Jacobian of Align rules have infinitesimally small approximation error. The proof of Proposition 1 proceeds by quantifying the discrepancy between the actual weight at each layer during training by gradient descent and the weight if the weight updates were computed with a fixed feedback Jacobian $\nabla_\theta f^{(0)}(x)$ instead of $\nabla_\theta f^{(t)}(x)$. As network width grows, this discrepancy is shown to approach $0$ at a rate faster than the rate at which weight changes approach $0$. Thus, in the infinite width limit, the weight changes behave \textit{as if} the feedback Jacobian were fixed at initialization. This means the weight change correlation matrix can be exactly equated to the time $0$ input correlation matrix.
\paragraph{Lemma 1. }
Assume layers are randomly initialized according to the neural tangent initialization~\cite{jacot2018NTK}. Also, assume that the nonlinearity $\sigma(\cdot)$ has bounded first and second derivatives. Suppose the network $f^{(t)}(x)$ is trained with Align-ada, Align-zero or gradient descent of learning rate $\eta$ on the mean-squared error loss for $T$ time such that $\max_x ||y(x)-f^{(t)}(x)||_2$ is uniformly bounded in $t \in [0, T]$. Denote $\theta$ as the parameters of the network, and for Align-zero and Align-ada, define $J^{(t)}$ as the approximation of $\nabla_\theta f^{(t)}(x)$ such that the corresponding network evolves as:
\begin{equation}
    \dot{f}^{(t)}(x) = \eta \nabla_\theta f^{(t)}(x)^T  \mathbb{E}[J^{(t)}(x) (y(x) - f^{(t)}(x))]
\end{equation}
Note that for Align-zero, $J^{(t)} = \nabla_\theta f^{(0)}(x)$.
For gradient descent, define $J^{(t)} = \nabla_\theta f^{(t)}(x)$. For Align-ada, $J^{(t)}$ is defined by first expanding $\nabla_\theta f^{(0)}(x)$ in terms of the weights and activations of the network, then replacing activations $z^{(0)}(x)$ at time $0$ with current activations $z^{(t)}(x)$. Then, in the simultaneous limit $\lim_{m_{N-1}, m_{N-2}, ... m_1 \to \infty}$:
\begin{align}
    \forall l=1, ... N-1, \sqrt{\mathbb{E}[||\frac{1}{\sqrt{m_l}} z^{(t)}_l(x) - \frac{1}{\sqrt{m_l}} z^{(0)}_l(x)||_2^2]}&\in \mathcal{O}(\frac{1}{\sqrt{\min \{ m_1, ... m_{N-1} \} }}) \\ \forall l=1, ... N, ||\frac{1}{\sqrt{m_{l-1}}} {W}_{l}^{(t)} - \frac{1}{\sqrt{m_{l-1}}} {W}_{l}^{(0)}||_{op} &\in \mathcal{O}(\frac{1}{\sqrt{\min \{ m_1, ... m_{N-1} \} }}) \\
    \max_x ||\nabla_\theta f^{(t)}(x) - \nabla_\theta f^{(0)}(x)||_{op} &\in \mathcal{O}(\frac{1}{\sqrt{\min \{ m_1, ... m_{N-1} \} }}) \\
    \max_x ||J^{(t)}(x) - \nabla_\theta f^{(0)}(x)||_{op} &\in \mathcal{O}(\frac{1}{\sqrt{\min \{ m_1, ... m_{N-1} \} }})
\end{align}

\begin{proof}
We adapt Lemma 1 of \cite{jacot2018NTK} to show that weights and intermediate layer activations of networks trained with Align-zero, Align-ada or gradient descent do not change over the course of training in the simultaneous limit $\lim_{m_{N-1}, m_{N-2}, ... m_1 \to \infty}$.

Define $a_l^{(t)} = \sqrt{\mathbb{E}[||\frac{1}{\sqrt{m_l}} z_l^{(t)}(x)||_2^2]}$, $w_l^{(t)} = ||\frac{1}{\sqrt{m_{l-1}}} {W}_{l}^{(t)}||_{op}$, $\tilde{a}_l^{(t)} = \sqrt{\mathbb{E}[||\frac{1}{\sqrt{m_l}} z^{(t)}_l(x) - \frac{1}{\sqrt{m_l}} z^{(0)}_l(x)||_2^2]}$, $\tilde{w}_l^{(t)} = ||\frac{1}{\sqrt{m_{l-1}}} {W}_{l}^{(t)} - \frac{1}{\sqrt{m_{l-1}}} {W}_{l}^{(0)}||_{op}$, where $W_l^{(t)}$ and $z_l^{(t)}(x)$ represent the weights and layers of the network trained with Align-zero, Align-ada or gradient descent. Note that the derivatives of $\tilde{w}_l^{(t)}$ can be bounded as:
\begin{equation}
    \dot{\tilde{w}}_l^{(t)} \leq ||\frac{1}{\sqrt{m_{l-1}}} \dot{W}_{l}^{(t)}||_{op} \leq \frac{1}{\sqrt{m_{l-1}}} \mathcal{P}(a_{l-1}^{(0)}, a_{l-1}^{(t)}, w^{(0)}_{l+1}, w^{(t)}_{l+1}, w^{(0)}_{l+2}, w^{(t)}_{l+2}, ... w^{(0)}_{N}, w^{(t)}_{N}) \max_x ||y(x) - f^{(t)}(x)||_2
\end{equation}
where $\mathcal{P}$ is a polynomial with bounded positive coefficients not depending on the width of any layers in the network. This is because weight changes at layer $l$ depend on the size of the previous layer ($a_{l-1}$) and the sizes of weights in later layers through the feedback process ($w_{l+1}, w_{l+2}, ... w_N$). Note that the coefficients of $\mathcal{P}$ can be bounded because $\sigma(\cdot)$ has a bounded first derivative. Similarly, the derivatives of $\tilde{a}_l^{(t)}$ can be bounded as:
\begin{equation}
    \dot{\tilde{a}}_l^{(t)} \leq \frac{1}{\sqrt{m_{l}}} \mathcal{P}(a_{0}^{(0)}, a_{0}^{(t)}, ... a_{l-1}^{(0)}, a_{l-1}^{(t)}, w_1^{(0)}, w_1^{(t)}, ... w^{(0)}_{N}, w^{(t)}_N) \max_x ||y(x) - f^{(t)}(x)||_2
\end{equation}
where $\mathcal{P}$ is another polynomial with bounded positive coefficients not depending on the width of any layers in the network. This is because the intermediate layer activations at layer $l$ change as the weights in earlier layers change, which depend on $w$ and $a$ as described above. Next, define $A(t)$ as:
\begin{equation}
    \sum_{l = 0}^{N-1} a_l^{(0)} + \tilde{a}_l^{(t)} + w_{l+1}^{(0)} + \tilde{w}_{l+1}^{(t)}
\end{equation}
Observe that using the previous bounds on $\dot{\tilde{w}}_l^{(t)}$ and $\dot{\tilde{a}}_l^{(t)}$, the derivative of $A(t)$ can be bounded as:
\begin{equation}
    \frac{d}{dt} A(t) \leq \frac{1}{\sqrt{\min \{ m_1, ... m_{N-1} \} }} \mathcal{P}(A(t)) \max_x ||y(x) - f^{(t)}(x)||_2
\end{equation}
where $\mathcal{P}$ is a polynomial with coefficients not depending on the width of any layers in the network. Finally, observe that $A(0)$ is stochastically bounded as it is only a function of $a_l^{(0)}$ and $w_l^{(0)}$. By Grönwall's inequality, we find that $A(t)$ can be uniformly bounded in an interval $[0, \tau]$, with $\tau$ approaching $T$ as $\min \{ m_1, ... m_{N-1} \} \to \infty$. This implies $\mathcal{P}(A(t))$ can be uniformly bounded in this interval. Recall also our assumption that $\max_x ||y(x) - f^{(t)}(x)||_2$ is also uniformly bounded in this interval. Then, due to the $\frac{1}{\sqrt{\min \{ m_1, ... m_{N-1} \} }}$ coefficient in the bound for $\frac{d}{dt} A(t)$, in the simultaneous limit $\lim_{m_{N-1}, m_{N-2}, ... m_1 \to \infty}$, $\frac{d}{dt} A(t)$ converges uniformly to $0$ at rate $\mathcal{O}(\frac{1}{\sqrt{\min \{ m_1, ... m_{N-1} \} }})$ in any interval $[0, \tau]$ where $\tau < T$. Thus, $A(t) = A(0)$ in this limit. Therefore, in the simultaneous limit $\lim_{m_{N-1}, m_{N-2}, ... m_1 \to \infty}$, $\tilde{a}_l^{(t)}$ and ${\tilde{w}}_l^{(t)}$ converge uniformly to $0$ in $[0, T]$ at rate $\mathcal{O}(\frac{1}{\sqrt{\min \{ m_1, ... m_{N-1} \} }})$.

Next, we prove that in the same limit,     
\begin{align}
\max_x ||\nabla_\theta f^{(t)}(x) - \nabla_\theta f^{(0)}(x)||_{op} &\in \mathcal{O}(\frac{1}{\sqrt{\min \{ m_1, ... m_{N-1} \} }}) \\
    \max_x ||J^{(t)}(x) - \nabla_\theta f^{(0)}(x)||_{op} &\in \mathcal{O}(\frac{1}{\sqrt{\min \{ m_1, ... m_{N-1} \} }})
\end{align}
First, observe that for $l=1,...N-1$, $\nabla_{z_l} f_i^{(t)}(x)$ can be expanded as:
    \begin{equation}
        \nabla_{z_l} f_i^{(t)}(x) = diag(\sigma'(z_{l-1}^{(t)}(x))) \frac{1}{\sqrt{m_{l-1}}} W^{(t)T}_{l} ... \frac{1}{\sqrt{m_{N-2}}} W^{(t)T}_{N-1} diag(\sigma'(z_{N-1}^{(t)}(x))) \frac{1}{\sqrt{m_{N-1}}} W^{(t)T}_{N, :, i}
    \end{equation}
    Similarly, $\nabla_{z_l} f_i^{(0)}(x)$ can be expanded as:
    \begin{equation}
        \nabla_{z_l} f_i^{(0)}(x) = diag(\sigma'(z_{l-1}^{(0)}(x))) \frac{1}{\sqrt{m_{l-1}}} W^{(0)T}_{l} ... \frac{1}{\sqrt{m_{N-2}}} W^{(0)T}_{N-1} diag(\sigma'(z_{N-1}^{(0)}(x))) \frac{1}{\sqrt{m_{N-1}}} W^{(0)T}_{N, :, i}
    \end{equation}
Thus, $\nabla_\theta f^{(t)}(x)$ and $J^{(t)}(x)$, which are formed of products of layerwise Jacobians $\nabla_{z} f(x)$ with intermediate layer activations $\sigma(z(x))$, can be written as sum of products of terms involving $z_l(x)$ and $\frac{1}{\sqrt{m_{l-1}}} W_{l}$ at time $t$ and time $0$. Specifically, since $\nabla_\theta f^{(t)}(x)$ is a function of weights and intermediate layer activations at time $t$:
\begin{multline}
    \max_x ||\nabla_\theta f^{(t)}(x) - \nabla_\theta f^{(0)}(x)||_{op} \leq \sum_{l=1}^{N} \tilde{w}_l^{(t)} \mathcal{P}_{w,l}(a_0^{(0)}, a_{0}^{(t)}, ... a_{N-1}^{(0)}, a_{N-1}^{(t)}, w_1^{(0)}, w_1^{(t)}, ... w^{(0)}_{N}, w^{(t)}_N) \\ + \sum_{l=0}^{N-1} \tilde{a}_l^{(t)} \mathcal{P}_{a,l}(a_0^{(0)}, a_{0}^{(t)}, ... a_{N-1}^{(0)}, a_{N-1}^{(t)}, w_1^{(0)}, w_1^{(t)}, ... w^{(0)}_{N}, w^{(t)}_N)
\end{multline}
where $\mathcal{P}_{w,l}$ and $\mathcal{P}_{a,l}$ are polynomials with positive coefficients not depending on any layer's width. This is because the time-variation of the gradient of any parameter in the network is constrained either by the variation in a weight $\tilde{w}_l^{(t)}$ or variation in a layer's activations $\tilde{a}_l^{(t)}$. Similarly, 
\begin{multline}
    \max_x ||J^{(t)}(x) - \nabla_\theta f^{(0)}(x)||_{op} \leq \sum_{l=1}^{N} \tilde{w}_l^{(t)} \mathcal{P}_{w,l}(a_0^{(0)}, a_{0}^{(t)}, ... a_{N-1}^{(0)}, a_{N-1}^{(t)}, w_1^{(0)}, w_1^{(t)}, ... w^{(0)}_{N}, w^{(t)}_N) \\ + \sum_{l=0}^{N-1} \tilde{a}_l^{(t)} \mathcal{P}_{a,l}(a_0^{(0)}, a_{0}^{(t)}, ... a_{N-1}^{(0)}, a_{N-1}^{(t)}, w_1^{(0)}, w_1^{(t)}, ... w^{(0)}_{N}, w^{(t)}_N)
\end{multline}
because for all three of Align-ada, Align-zero and gradient descent, $J^{(t)}(x)$ is a function of weights and intermediate layer activations at time $t$ or time $0$. Note that the left hand size is $0$ for Align-zero since $J^{(t)}(x) = \nabla_\theta f^{(0)}(x)$. Using the boundedness of $A(t)$, observe that $\max_x ||J^{(t)}(x) - \nabla_\theta f^{(0)}(x)||_{op}$ and $\max_x ||\nabla_\theta f^{(t)}(x) - \nabla_\theta f^{(0)}(x)||_{op}$ converge to $0$ at the same rate as $\tilde{W}_l^{(t)}$ and $\tilde{a}_l^{(t)}$. Therefore,
\begin{align}
\max_x ||\nabla_\theta f^{(t)}(x) - \nabla_\theta f^{(0)}(x)||_{op} &\in \mathcal{O}(\frac{1}{\sqrt{\min \{ m_1, ... m_{N-1} \} }}) \\
    \max_x ||J^{(t)}(x) - \nabla_\theta f^{(0)}(x)||_{op} &\in \mathcal{O}(\frac{1}{\sqrt{\min \{ m_1, ... m_{N-1} \} }})
\end{align}
\end{proof}

Next, we prove Proposition 1:
\paragraph{Proposition 1.} Assume layers are randomly initialized according to the neural tangent initialization~\cite{jacot2018NTK}. Also, assume that the nonlinearity $\sigma(\cdot)$ has bounded first and second derivatives. Suppose the network is trained with continuous gradient flow with learning rate $\eta$ on the mean-squared error loss for $T$ time such that $\max_x ||y(x)-f^{(t)}(x)||_2$ is uniformly bounded in $t \in [0, T]$. Then, in the simultaneous limit $\lim_{m_{N-1}, m_{N-2}, ... m_1 \to \infty}$:
    \begin{equation}
        \frac{1}{\sqrt{m_{l-1}}} || \frac{1}{m_{l-1}} \Sigma^{(0)}_{l,q^{(t)}_l} -  \Delta^{(t)}_l ||_{op} \in \mathcal{O}(\frac{1}{\min \{ m_1, ... m_{N-1} \} })
    \end{equation}
\begin{proof}
    First, by Lemma 1, observe that in the simultaneous limit $\lim_{m_{N-1}, m_{N-2}, ... m_1 \to \infty}$:
    \begin{align}
    \sqrt{\mathbb{E}[||\frac{1}{\sqrt{m_l}} z^{(t)}_l(x) - \frac{1}{\sqrt{m_l}} z^{(0)}_l(x)||_2^2]}&\in \mathcal{O}(\frac{1}{\sqrt{\min \{ m_1, ... m_{N-1} \} }}) \\ ||\frac{1}{\sqrt{m_{l-1}}} {W}_{l}^{(t)} - \frac{1}{\sqrt{m_{l-1}}} {W}_{l}^{(0)}||_{op} &\in \mathcal{O}(\frac{1}{\sqrt{\min \{ m_1, ... m_{N-1} \} }}) \\
    \max_x ||\nabla_\theta f^{(t)}(x) - \nabla_\theta f^{(0)}(x)||_{op} &\in \mathcal{O}(\frac{1}{\sqrt{\min \{ m_1, ... m_{N-1} \} }})
    \end{align}
    Since $\nabla_{z_l} f(x)$ is a component of $\nabla_\theta f(x)$, corresponding to bias parameters, this result implies that $\mathbb{E}_{p_x} [ ||\nabla_{z_l} f^{(t)}(x) - \nabla_{z_l} f^{(0)}(x)||_{op}]$ also converges to $0$ at rate $\mathcal{O}(\frac{1}{\sqrt{\min \{ m_1, ... m_{N-1} \} }})$.

    Next, we define $\tilde{W}^{(t)}$ as:
    \begin{equation}
        \tilde{W}_l^{(t)} = W_l^{(t)} - W_l^{(0)} - \frac{1}{\sqrt{m_{l-1}}} \eta \mathbb{E}_{p_x}[\nabla_{z_l} f^{(0)}(x) [\int_{0}^t (y(x)-f^{(\tau)}(x)) d \tau] \sigma(z_{l-1}^{(0)}(x))^T]
    \end{equation}
    Note that $\tilde{W}_l^{(0)}= 0$. Taking the derivative of both sides:
    \begin{multline}
        \dot{\tilde{W}}_l^{(t)} = \frac{1}{\sqrt{m_{l-1}}} \eta \mathbb{E}_{p_x}[\nabla_{z_l} f^{(t)}(x) (y(x)-f^{(t)}(x)) \sigma(z_{l-1}^{(t)}(x))^T] \\ - \frac{1}{\sqrt{m_{l-1}}} \eta \mathbb{E}_{p_x}[\nabla_{z_l} f^{(0)}(x) (y(x)-f^{(t)}(x)) \sigma(z_{l-1}^{(0)}(x))^T]
        \\= \frac{1}{\sqrt{m_{l-1}}} \eta \mathbb{E}_{p_x}[\nabla_{z_l} f^{(t)}(x) (y(x)-f^{(t)}(x)) (\sigma(z_{l-1}^{(t)}(x)) - \sigma(z_{l-1}^{(0)}(x)))^T]
        \\ + \frac{1}{\sqrt{m_{l-1}}} \eta \mathbb{E}_{p_x}[(\nabla_{z_l} f^{(t)}(x) - \nabla_{z_l} f^{(0)}(x)) (y(x)-f^{(t)}(x)) \sigma(z_{l-1}^{(0)}(x))^T]
    \end{multline}
Bounding the norm of $\dot{\tilde{W}}_l^{(t)}$:
\begin{multline}
    ||\dot{\tilde{W}}_l^{(t)}||_{op} \leq \frac{1}{\sqrt{m_{l-1}}} \eta \mathbb{E}_{p_x}[||\nabla_{z_l} f^{(t)}(x) (y(x)-f^{(t)}(x)) (\sigma(z_{l-1}^{(t)}(x)) - \sigma(z_{l-1}^{(0)}(x)))^T||_{op}]
    \\ + \frac{1}{\sqrt{m_{l-1}}} \eta \mathbb{E}_{p_x}[||(\nabla_{z_l} f^{(t)}(x) - \nabla_{z_l} f^{(0)}(x)) (y(x)-f^{(t)}(x)) \sigma(z_{l-1}^{(0)}(x))^T||_{op}]
    \\ \leq \frac{1}{\sqrt{m_{l-1}}} \eta \mathbb{E}_{p_x}[||\nabla_{z_l} f^{(t)}(x)||_{op} ||y(x)-f^{(t)}(x)||_2 ||\sigma(z_{l-1}^{(t)}(x)) - \sigma(z_{l-1}^{(0)}(x))||_2]
    \\ + \frac{1}{\sqrt{m_{l-1}}} \eta \mathbb{E}_{p_x}[||\nabla_{z_l} f^{(t)}(x) - \nabla_{z_l} f^{(0)}(x)||_{op} ||y(x)-f^{(t)}(x)||_2 ||\sigma(z_{l-1}^{(0)}(x))||_{2}]
    \\ \leq c \eta \max_x ||\nabla_{z_l} f^{(t)}(x)||_{op} \max_x||y(x)-f^{(t)}(x)||_2 \mathbb{E}_{p_x} [||\frac{1}{\sqrt{m_{l-1}}} (z_{l-1}^{(t)}(x) - z_{l-1}^{(0)}(x))||_2]
    \\ + \eta \mathbb{E}_{p_x} [ ||\nabla_{z_l} f^{(t)}(x) - \nabla_{z_l} f^{(0)}(x)||_{op}] \max_x ||y(x)-f^{(t)}(x)||_2 \max_x ||\frac{1}{\sqrt{m_{l-1}}}\sigma(z_{l-1}^{(0)}(x))||_{2}
\end{multline}
where $c$ is the Lipschitz constant of $\sigma$. Next, observe that $\max_{t \in [0, T]} \max_x ||\nabla_{z_l} f^{(t)}(x)||_{op}$ and $\max_{t \in [0, T]} \max_x ||\frac{1}{\sqrt{m_{l-1}}}\sigma(z_{l-1}^{(0)}(x))||_{2}$ are bounded in the simultaneous limit $\lim_{m_{N-1}, m_{N-2}, ... m_1 \to \infty}$.  Moreover, $\mathbb{E}_{p_x} [||\frac{1}{\sqrt{m_{l-1}}} (z_{l-1}^{(t)}(x) - z_{l-1}^{(0)}(x))||_2]$ and $\mathbb{E}_{p_x} [ ||\nabla_{z_l} f^{(t)}(x) - \nabla_{z_l} f^{(0)}(x)||_{op}]$ approach $0$ in this limit at rate $\mathcal{O}(\frac{1}{\sqrt{\min \{ m_1, ... m_{N-1} \} }})$. Since $\max_x ||y(x)-f^{(t)}(x)||_2$ is uniformly bounded in $t \in [0, T]$ by assumption, $\max_{t \in [0, T]} ||\dot{\tilde{W}}_l^{(t)}||_{op}$ also approaches $0$ in this limit at rate $\mathcal{O}(\frac{1}{\sqrt{\min \{ m_1, ... m_{N-1} \} }})$. Thus,
\begin{multline}
   \lim_{m_{N-1}, m_{N-2}, ... m_1 \to \infty}  ||\tilde{W}_l^{(t)}||_{op} \leq  \lim_{m_{N-1}, m_{N-2}, ... m_1 \to \infty}  \int_0^t ||\dot{\tilde{W}}_l^{(\tau)}|| d \tau \\ \leq \int_0^t \lim_{m_{N-1}, m_{N-2}, ... m_1 \to \infty} ||\dot{\tilde{W}}_l^{(\tau)}|| d \tau \in \mathcal{O}(\frac{1}{\sqrt{\min \{ m_1, ... m_{N-1} \} }})
\end{multline}

Next, observe that:
\begin{multline}
    \Sigma^{(0)}_{l,q^{(t)}_l} = \mathbb{E}_{x_1\sim p_x, x_2\sim p_x}[ \sigma(z_{l-1}^{(0)}(x_1)) \delta^{(t)}(x_1)^T \nabla_{z_l} f^{(0)}(x_1)^T \nabla_{z_l} f^{(0)}(x_2) \delta^{(t)}(x_2) \sigma(z_{l-1}^{(0)}(x_2))^T] \\= m_{l-1} (W_l^{(t)}-W_l^{(0)}-\tilde{W}_l^{(0)})^T (W_l^{(t)}-W_l^{(0)}-\tilde{W}_l^{(0)})
\end{multline}
This implies:
\begin{multline}
    \frac{1}{\sqrt{m_{l-1}}} || \frac{1}{m_{l-1}} \Sigma^{(0)}_{l,q^{(t)}_l} -  \Delta^{(t)}_l ||_{op} \\ =  \frac{1}{\sqrt{m_{l-1}}}|| (W_l^{(t)}-W_l^{(0)}-\tilde{W}_l^{(t)})^T (W_l^{(t)}-W_l^{(0)}-\tilde{W}_l^{(t)}) - (W_l^{(t)}-W_l^{(0)})^T (W_l^{(t)}-W_l^{(0)})||_{op}
    \\ \leq 2 \frac{1}{\sqrt{m_{l-1}}}|| (W_l^{(t)}-W_l^{(0)})||_{op} ||\tilde{W}_l^{(t)}||_{op} + \frac{1}{\sqrt{m_{l-1}}}||\tilde{W}_l^{(t)T} \tilde{W}_l^{(0)}||_{op}
\end{multline}
Finally, note that $\frac{1}{\sqrt{m_{l-1}}}|| (W_l^{(t)}-W_l^{(0)})||_{op}$ converges to $0$ at rate $\mathcal{O}(\frac{1}{\sqrt{\min \{ m_1, ... m_{N-1} \} }})$. Combining this with the convergence of $||\tilde{W}_l^{(t)}||_{op}$ to $0$, we find that $\frac{1}{\sqrt{m_{l-1}}} || \frac{1}{m_{l-1}} \Sigma^{(0)}_{l,q^{(t)}_l} -  \Delta^{(t)}_l ||_{op}$ converges to $0$ at rate $\mathcal{O}(\frac{1}{\min \{ m_1, ... m_{N-1} \} })$.
\end{proof}
\paragraph{Interpretation of the width scaling in Proposition 1}
Observe that as network width approaches infinity, both weights and intermediate layer activations change vanishingly over the course of training. Thus, the corresponding correlation matrices $\Delta_l^{(t)}$ and $\Sigma^{(0)}_{l, q_l^{(t)}}$ also approach zero under the appropriate width-dependent normalizations. In light of this, it is important to consider whether Proposition 1 really shows an alignment effect between $\Delta_l^{(t)}$ and $\Sigma^{(0)}_{l, q_l^{(t)}}$ or merely holds because the correlation matrices individually approach zero.

Note that since $||W_l^{(t)} - W_l^{(0)}||_{op}$ grows as $o(\sqrt{m_{l-1}})$, the weight change correlation $||\Delta_l^{(t)}||_{op}$ grows as $o(m_{l-1})$. Thus, $\frac{1}{\sqrt{m_{l-1}}}||\Delta_l^{(t)}||_{op}$ grows as $o(\sqrt{m_{l-1}})$, which will generally \textit{not} converge to zero. Similarly, activated layer norms $||\sigma(z_{l-1}^{(0)}(x))||_2$ grow as $O(m_{l-1})$, implying $||\Sigma^{(0)}_{l, q_l^{(t)}}||_{op}$ grows as $O(m_{l-1}^2)$. Thus, $\frac{1}{\sqrt{m_{l-1}}}||\frac{1}{m_{l-1}} \Sigma^{(0)}_{l, q_l^{(t)}}||_{op}$ grows as $O(\sqrt{m_{l-1}})$, which also will not converge to zero in general. Thus, neither term in the Proposition 1 expression individually approaches zero, and the result does not trivially follow from the individual terms separately approaching zero. Therefore, the convergence of the norm difference in Proposition 1 demonstrates that the two correlation matrices point in the same direction away from the origin and indeed can be interpreted as an alignment effect.

\subsection{Proof of Proposition~\ref{prop: inf_width_linearized}} \label{prop2_proof}
The proof of Proposition 2 starts by quantifying the difference between the output of an Align-trained network and a backprop-trained network during training. This difference is shown to be bounded by the discrepancy between the feedback Jacobian $\nabla_\theta f^{(t)}(x)$ of the backprop network, its true value in the Align-trained network, and its approximation in the Align-trained network. By Lemma 1, these discrepancies approach $0$ in infinite-width networks; thus, Align rules approach gradient descent in the infinite width limit.
\paragraph{Proposition 2} Suppose networks $f^{(t)}(x)$, $f_{al\cdot 0}^{(t)}(x)$ and $f_{al\cdot ada}^{(t)}(x)$ are trained with gradient descent, Align-zero and Align-ada respectively and share the same initialization. Assume layers are randomly initialized according to the neural tangent initialization~\cite{jacot2018NTK}. Also, assume that the nonlinearity $\sigma(\cdot)$ has bounded first and second derivatives. Suppose all methods use learning rate $\eta$ on the mean-squared error loss for $T$ time such that $\max_x ||y(x)-f^{(t)}(x)||_2$ is uniformly bounded in $t \in [0, T]$. Then, in the simultaneous limit $\lim_{m_{N-1}, m_{N-2}, ... m_1 \to \infty}$:
\begin{equation}
    ||f^{(t)}(x) - f_{al\cdot 0}^{(t)}(x)||,\,\, ||f^{(t)}(x) - f_{al\cdot ada}^{(t)}(x)||  \in \mathcal{O}(\frac{1}{\sqrt{\min \{ m_1, ... m_{N-1} \} }})
\end{equation}
\begin{proof}
First, note that $f^{(t)}(x)$ evolves as:
\begin{equation}
    \dot{f}^{(t)}(x) = \eta \nabla_\theta f^{(t)}(x)^T  \mathbb{E}[\nabla_\theta f^{(t)}(x) (y(x) - f^{(t)}(x))]
\end{equation}
where $\theta$ are the parameters of $f^{(t)}$. The term in expectation corresponds to the gradient of the total loss (over all points) with respect to the network parameters $\theta$, while $\nabla_\theta f^{(t)}(x)^T$ indicates how parameter changes affect the value of $f^{(t)}$ evaluated at point $x$.

We consider a general alternative learning procedure that trains a network $f_*^{(t)}$ which replaces the feedback Jacobian $\nabla_\theta f^{(t)}(x)$ with a general function $J_*^{(t)}(x)$. $f_*^{(t)}$ evolves as:
\begin{equation}
    \dot{f}_*^{(t)}(x) = \eta \nabla_\theta f^{(t)}_*(x)^T  \mathbb{E}[J_*^{(t)}(x) (y(x) - f_*^{(t)}(x))]
\end{equation}
Note that here, parameter changes, corresponding to the term in expectation, depend on $J_*^{(t)}(x)$ and not on the true feedback Jacobian $\nabla_\theta f^{(t)}_*(x)$. Thus, in general, the above rule avoids exact backpropagation. However, parameter changes affect the network output according to the true feedback Jacobian $\nabla_\theta f^{(t)}_*(x)$.

Define $\Delta_1^{(t)}(x) = \nabla_\theta f^{(t)}_*(x) - \nabla_\theta f^{(t)}(x)$ and $\Delta_2^{(t)}(x) = J_*^{(t)}(x) - \nabla_\theta f^{(t)}(x)$. Next, finding the difference of $\dot{f}^{(t)}(x)$ and $\dot{f}_*^{(t)}(x)$:
\begin{multline}
    \dot{f}_*^{(t)}(x)-\dot{f}^{(t)}(x) = \eta \nabla_\theta f^{(t)}_*(x)^T  \mathbb{E}[J_*^{(t)}(x) (y(x) - f_*^{(t)}(x))] - \eta \nabla_\theta f^{(t)}(x)^T  \mathbb{E}[\nabla_\theta f^{(t)}(x) (y(x) - f^{(t)}(x))]
    \\  = \eta \nabla_\theta f^{(t)}_*(x)^T  \mathbb{E}[J_*^{(t)}(x) (y(x) - f^{(t)}(x))] - \eta \nabla_\theta f^{(t)}(x)^T  \mathbb{E}[\nabla_\theta f^{(t)}(x) (y(x) - f^{(t)}(x))] \\ + \eta \nabla_\theta f^{(t)}_*(x)^T \mathbb{E}[J_*^{(t)}(x) (f^{(t)}(x) - f_*^{(t)}(x))]
    \\  = \eta \nabla_\theta f^{(t)}(x)^T  \mathbb{E}[J_*^{(t)}(x) (y(x) - f^{(t)}(x))] - \eta \nabla_\theta f^{(t)}(x)^T  \mathbb{E}[\nabla_\theta f^{(t)}(x) (y(x) - f^{(t)}(x))] \\ + \eta \Delta_1^{(t)}(x)^T  \mathbb{E}[J_*^{(t)}(x) (y(x) - f^{(t)}(x))] + \eta \nabla_\theta f^{(t)}_*(x)^T \mathbb{E}[J_*^{(t)}(x) (f^{(t)}(x) - f_*^{(t)}(x))]
    \\  = \eta \Delta_1^{(t)}(x)^T  \mathbb{E}[(\nabla_\theta f^{(t)}(x) + \Delta_2^{(t)}(x)) (y(x) - f^{(t)}(x))] + \eta \nabla_\theta f^{(t)}(x)^T  \mathbb{E}[ \Delta_2^{(t)}(x) (y(x) - f^{(t)}(x))]  \\+ \eta  (\nabla_\theta f^{(t)}(x)  + \Delta_1^{(t)}(x))^T  \mathbb{E}[(\nabla_\theta f^{(t)}(x) + \Delta_2^{(t)}(x)) (f^{(t)}(x) - f_*^{(t)}(x))]
\end{multline}
Bounding the derivative of the discrepancy between ${f}^{(t)}(x)$ and ${f}_*^{(t)}(x)$:
\begin{multline}
    \frac{d}{dt} ||{f}_*^{(t)}(x)-{f}^{(t)}(x)||_2 \leq ||\dot{f}_*^{(t)}(x)-\dot{f}^{(t)}(x)||_2 
    \\  \leq \eta \max_x ||\Delta_1^{(t)}(x)||_{op}  (\max_x ||\nabla_\theta f^{(t)}(x)||_{op} + \max_x  ||\Delta_2^{(t)}(x)||_{op}) \max_x ||y(x) - f^{(t)}(x)||_2 \\ + \eta \max_x ||\nabla_\theta f^{(t)}(x)||_{op}  \max_x ||\Delta_2^{(t)}(x)||_{op} \max_x ||y(x) - f^{(t)}(x)||_2  \\+ \eta  (\max_x ||\nabla_\theta f^{(t)}(x)||_{op}  + \max_x ||\Delta_1^{(t)}(x)||_{op})  (\max_x ||\nabla_\theta f^{(t)}(x)||_{op} +  \max_x ||\Delta_2^{(t)}(x)||_{op}) \max_x ||f^{(t)}(x) - f_*^{(t)}(x)||_2
\end{multline}
Taking the integral of both sides:
\begin{multline}
    ||{f}_*^{(t)}(x)-{f}^{(t)}(x)||_2 
    \\  \leq \int_0^t \eta \max_x ||\Delta_1^{(\tau)}(x)||_{op}  (\max_x ||\nabla_\theta f^{(\tau)}(x)||_{op} + \max_x  ||\Delta_2^{(\tau)}(x)||_{op}) \max_x ||y(x) - f^{(\tau)}(x)||_2 d \tau \\ + \int_0^t \eta \max_x ||\nabla_\theta f^{(\tau)}(x)||_{op}  \max_x ||\Delta_2^{(\tau)}(x)||_{op} \max_x ||y(x) - f^{(\tau)}(x)||_2 d \tau \\+ \eta  (\max_x ||\nabla_\theta f^{(t)}(x)||_{op}  + \max_x ||\Delta_1^{(t)}(x)||_{op})  (\max_x ||\nabla_\theta f^{(t)}(x)||_{op} +  \max_x ||\Delta_2^{(t)}(x)||_{op}) \max_x ||f^{(t)}(x) - f_*^{(t)}(x)||_2
\end{multline}
Define 
\begin{multline}
\alpha(t) = \int_0^t \eta \max_x ||\Delta_1^{(\tau)}(x)||_{op}  (\max_x ||\nabla_\theta f^{(\tau)}(x)||_{op} + \max_x  ||\Delta_2^{(\tau)}(x)||_{op}) \max_x ||y(x) - f^{(\tau)}(x)||_2 d \tau \\ + \int_0^t \eta \max_x ||\nabla_\theta f^{(\tau)}(x)||_{op}  \max_x ||\Delta_2^{(\tau)}(x)||_{op} \max_x ||y(x) - f^{(\tau)}(x)||_2 d \tau
\end{multline}
\begin{equation}
    \beta(t) = \eta  (\max_x ||\nabla_\theta f^{(t)}(x)||_{op}  + \max_x ||\Delta_1^{(t)}(x)||_{op})  (\max_x ||\nabla_\theta f^{(t)}(x)||_{op} +  \max_x ||\Delta_2^{(t)}(x)||_{op})
\end{equation}
\begin{equation}
    u(t) = \max_x ||{f}_*^{(t)}(x)-{f}^{(t)}(x)||_2
\end{equation}
Then, the above inequality implies:
\begin{equation}
    u(t) \leq \alpha(t) + \int_0^t \beta(\tau) u(\tau) d\tau
\end{equation}
By Grönwall's inequality:
\begin{equation}
    u(t) \leq \alpha(t) \exp{\left(\int_0^t \beta(\tau) d\tau\right)}
\end{equation}
Finally, observe that if in the simultaneous limit $\lim_{m_{N-1}, m_{N-2}, ... m_1 \to \infty}$, $\max_x ||\Delta_1^{(t)}(x)||_{op}$ and $\max_x ||\Delta_2^{(t)}(x)||_{op}$ converge uniformly to $0$ in range $[0,T]$, then, in the same limit, $\alpha(t)$ converges to $0$ at the same rate implying $u(t)$ converges to $0$ at the same rate.

Finally, observe that by Lemma 1, $\max_x ||\Delta_1^{(t)}(x)||_{op}$ and $\max_x ||\Delta_2^{(t)}(x)||_{op}$ converge to $0$ at rate $\mathcal{O}(\frac{1}{\sqrt{\min \{ m_1, ... m_{N-1} \} }})$. Thus, $u(t)$ converges to $0$ at rate $\mathcal{O}(\frac{1}{\sqrt{\min \{ m_1, ... m_{N-1} \} }})$, yielding the result:
\begin{equation}
    ||f^{(t)}(x) - f_{al\cdot 0}^{(t)}(x)||,\,\, ||f^{(t)}(x) - f_{al\cdot ada}^{(t)}(x)||  \in \mathcal{O}(\frac{1}{\sqrt{\min \{ m_1, ... m_{N-1} \} }})
\end{equation}
\end{proof}
\subsection{Pseudocode for Align-zero and Align-ada} \label{app: pseudocode}
\begin{algorithm}[H]
\begin{algorithmic}
\REQUIRE Training points $\mathcal{X}$, training steps $T$, layer widths $m_1, m_2, ... m_L$, learning rate $\eta$
\STATE Initialize weights $W^{(0)}_1, b^{(0)}_1, ... W^{(0)}_N, b^{(0)}_N$

\FOR{Mini-batch $x \in \mathcal{X}$} 
    \FOR{$l = 1, ... N$}
        \STATE $z_{l}^{(0)}(x)  = \frac{1}{\sqrt{m_{l-1}}} W^{(0)}_l \sigma(z_{l-1}^{(0)}(x)) + b_l^{(0)}$
    \ENDFOR
    \STATE $g_N(x) = I$ // Finding $\nabla_{z_l} f^{(0)}(x)$
    
    \FOR{$l = N-1, ... 1$}
        \STATE $g_l(x) = diag(\sigma'(z_{l}^{(0)}(x))) \frac{1}{\sqrt{m_{l}}} W_{l+1}^T g_{l+1}(x)$
    \ENDFOR
\ENDFOR

\FOR{$t = 1, ... T$}
    \STATE Sample mini-batch $x \in \mathcal{X}$
    
    \FOR{$l = 1, ... N$}
        \STATE $z_{l}^{(t-1)}(x)  = \frac{1}{\sqrt{m_{l-1}}} W^{(t-1)}_l \sigma(z_{l-1}^{(t-1)}(x)) + b_l^{(t-1)}$
    \ENDFOR
    
    \STATE \hl{${W}_{l}^{(t)} =  W_l^{(t-1)} + \frac{\eta}{\sqrt{m_{l-1}}} g_l(x) \nabla_{z_N(x)} L(z_N^{(t-1)}(x), y(x)) \sigma(z_{l-1}^{(0)}(x))^T$}
    
    \STATE ${b}_{l}^{(t)} =  b_l^{(t-1)} + \eta g_l(x) \nabla_{z_N(x)} L(z_N^{(t-1)}(x), y(x))$
\ENDFOR
\STATE Return $W^{(T)}_1, b^{(T)}_1, ... W^{(T)}_N, b^{(T)}_N$
\end{algorithmic}
 \caption{Align-zero training}
 \label{algo: linearized}
\end{algorithm}

\begin{algorithm}[H]
\begin{algorithmic}
\REQUIRE Training points $\mathcal{X}$, training steps $T$, layer widths $m_1, m_2, ... m_L$, learning rate $\eta$
\STATE Initialize weights $W^{(0)}_1, b^{(0)}_1, ... W^{(0)}_N, b^{(0)}_N$

\FOR{Mini-batch $x \in \mathcal{X}$} 
    \FOR{$l = 1, ... N$}
        \STATE $z_{l}^{(0)}(x)  = \frac{1}{\sqrt{m_{l-1}}} W^{(0)}_l \sigma(z_{l-1}^{(0)}(x)) + b_l^{(0)}$
    \ENDFOR
    \STATE $g_N(x) = I$ // Finding $\nabla_{z_l} f^{(0)}(x)$
    
    \FOR{$l = N-1, ... 1$}
        \STATE $g_l(x) = diag(\sigma'(z_{l}^{(0)}(x))) \frac{1}{\sqrt{m_{l}}} W_{l+1}^T g_{l+1}(x)$
    \ENDFOR
\ENDFOR

\FOR{$t = 1, ... T$}
    \STATE Sample mini-batch $x \in \mathcal{X}$
    
    \FOR{$l = 1, ... N$}
        \STATE $z_{l}^{(t-1)}(x)  = \frac{1}{\sqrt{m_{l-1}}} W^{(t-1)}_l \sigma(z_{l-1}^{(t-1)}(x)) + b_l^{(t-1)}$
    \ENDFOR
    
    \STATE \hl{${W}_{l}^{(t)} =  W_l^{(t-1)} + \frac{\eta}{\sqrt{m_{l-1}}} g_l(x) \nabla_{z_N(x)} L(z_N^{(t-1)}(x), y(x)) \sigma(z_{l-1}^{(t-1)}(x))^T$}
    
    \STATE ${b}_{l}^{(t)} =  b_l^{(t-1)} + \eta g_l(x) \nabla_{z_N(x)} L(z_N^{(t-1)}(x), y(x))$
\ENDFOR
\STATE Return $W^{(T)}_1, b^{(T)}_1, ... W^{(T)}_N, b^{(T)}_N$
\end{algorithmic}
 \caption{Align-ada training}
 \label{algo: semilinearized}
\end{algorithm}

\subsection{Remark on difference between Align-zero and~\cite{lee2019WideNN}} \label{remark} Our Align-zero learning rule differs subtly but importantly from the linearized networks of~\cite{lee2019WideNN}. Both formulations use the learning dynamics of Equation~\eqref{lin_align_rule} for their parameters -- which is linear in the parameters. However, the two formulations apply Equation~\eqref{lin_align_rule} as a function of their respective activation dynamics $f^{(t)}_{al\cdot 0}(x), {f}^{(t)}_{lin}(x)$. While our activation dynamics $f^{(t)}_{al\cdot 0}(x)$ are a nonlinear neural network with parameters ${W}_{l, al\cdot 0}^{(t)}$ and ${b}_{l, al\cdot 0}^{(t)}$, the activation dynamics ${f}^{(t)}_{lin}(x)$ of~\cite{lee2019WideNN} are themselves linear and \textit{not a neural network}. Rather, ${f}^{(t)}_{lin}(x)$ is the first order term in the Taylor expansion of the neural network ${f}^{(t)}_{al\cdot 0}(x)$:
\begin{equation}
{f}_{al\cdot 0}^{(t)}(x) =  f^{(0)}(x) + \nabla_\theta f^{(0)}(x)^T  (\theta^{(t)} - \theta^{(0)})  + O((\theta^{(t)} - \theta^{(0)})^2)\\ = {f}_{lin}^{(t)}(x) + O((\theta^{(t)} - \theta^{(0)})^2)
\end{equation}
where $\theta^{(t)}$ represents the parameters of both functions. While the two types of linearization behave similarly near initialization, our approach linearizes only the parameter dynamics, while the approach of~\cite{lee2019WideNN} also linearizes the network's activation dynamics. We focus on only linearizing the parameter dynamics because our goal is to study learning rules for actual (non-linearized) neural networks.

\subsection{Implementation of Align-zero and Align-ada in recurrent neural networks} \label{app: rnn_impl}
Although throughout this paper we use the notation of feedforward neural networks, note that our analysis applies to recurrent neural networks as well since RNNs produce feedforward networks when unrolled over time. For example, in the notation of feedforward networks, $z_l(x)$ would refer to the pre-activated hidden state of the network at time step $l$ of recurrent processing. However, there are typically three key differences with the feedforward case: 1) losses $L(\hat{y}_l(x), y_l(x))$ can be accumulated over different time steps where $\hat{y}_l(x)$ is the predicted label depending on state $z_l(x)$, 2) inputs can be fed in at different time steps, 3) that recurrent weights and biases are tied at all time-steps (i.e. $W = W_1 = W_2= ...; b = b_1 = b_2 =...$), although these weights can change over the course of training. Thus, parameter updates must be summed over all losses and all recurrent time-steps where the parameters are applied. The resulting Align-ada learning rule for the recurrent weights of an RNN is:
\begin{equation} \label{slin_rule_rnn}
    \dot{W}_{al\cdot ada}^{(t)} =   \frac{\eta}{\sqrt{m}} \times \\ 
    \sum_{i, j: j > i} \mathbb{E}_{p_x}[\nabla_{z_i} \hat{y}_j^{(0)}(x) \nabla_{\hat{y}_j} L(\hat{y}_j^{(t)}(x), y_j(x)) \sigma(z_{i-1}^{(t)}(x))^T ]
\end{equation}
The learning rules for other parameters  (i.e. biases, readout-weights, etc.) can be found similarly, and the adaptation of the Align-zero learning rule can be found analogously.

\subsection{Implementation details} \label{app: hyperparams}
All experiments are run on an Nvidia Tesla K20Xm GPU. Unless otherwise mentioned, all networks are trained under the NTK parameterization as defined in Equation~\ref{eqn: ntk_param}; see Appendix~\ref{app: std_param} for experiments under standard parameterization. For alignment measurements, we compute alignment scores using Equation~\eqref{eqn: align_score}, estimating expectations with 100 Gaussian samples. For all randomness, the random seed is set to $99$.

\paragraph{Baselines}
We primarily compare Align-ada and Align-zero against biologically-motivated learning algorithms, although, for reference, we also compare against gradient descent (denoted Normal). Among biologically-motivated algorithms, we compare with training only the last fully connected layer of the network, leaving intermediate layer parameters fixed, which provides a useful baseline to check that other methods are optimizing intermediate layer parameters (denoted Last layer). We also compare with feedback alignment (denoted FA) \cite{lillicrap2016random} which uses fixed random feedback weights; and with direct feedback alignment (denoted DFA) \cite{nokland2019training}, which replaces the $x$-dependent feedback Jacobian $\nabla_{z_l} f^{(0)}(x)$ of Align-zero with a constant random matrix.

\paragraph{CIFAR-10 experiments}
We train on CNN architecture with 7 convolutional layers followed by global average pooling and a fully connected layer. All convolutional layers use 3-by-3 convolutions with the same number of filters, which is varied from 8 to 512. All but the 4th and 6th convolutional layers have stride 1-by-1, with the other two using stride 2-by-2. All convolution layers are followed by batch normalization and a ReLU non-linearity, unless otherwise mentioned. Networks are trained on a mean squared error loss. Following~\cite{novak2019bayesian}, labels are constructed as one hot encoded labels minus $0.1$ so that they have mean $0$. Images are scaled in range $[0,1]$. For each training method, test set accuracies are evaluated at each training epoch, and the best results reported. A learning rate grid search in $\{1, 2, 5\}$ is performed for all training methods and the best result reported unless unless otherwise mentioned. See our experiments varying learning rate for observations on training instability beyond a learning rate of $5$. We use a batch size of 100, and all methods are trained for 150 epochs unless otherwise mentioned. Best results are reported over all epochs trained unless otherwise reported. For experiments on Small CIFAR-10, we randomly sample a subset of points from the full CIFAR-10 training set to use as training points.

\paragraph{KMNIST experiments}
We train on CNN architecture with 3 convolutional layers followed by global average pooling and a fully connected layer. All convolutional layers use 3-by-3 convolutions with the same number of filters, which is varied from 8 to 512. The first layer uses stride 1-by-1, with the other two using stride 2-by-2. All convolution layers are followed by batch normalization and a ReLU non-linearity. Networks are trained on a mean squared error loss. Following~\cite{novak2019bayesian}, labels are constructed as one hot encoded labels minus $0.1$ so that they have mean $0$. Images are scaled in range $[0,1]$. For each training method, test set accuracies are evaluated at each training epoch, and the best results reported. A learning rate grid search in $\{1, 2, 5\}$ is performed for all training methods and the best result reported unless unless otherwise mentioned. See our experiments varying learning rate for observations on training instability beyond a learning rate of $5$. We use a batch size of 100, and all methods are trained for 100 epochs unless otherwise mentioned. Best results are reported over all epochs trained unless otherwise reported.

\paragraph{Add task experiments}
The task has a sequence of iid Bernoulli inputs $x(t)$ and desired labels $y(t)$ constructed as:
\begin{equation}
    y(t) = 0.5 + 0.5 x(t) * \delta(t-2) - 0.25 x(t) * \delta(t-5) 
\end{equation}
where $*$ denotes convolution. The dataset consists of $400$ labeled input sequences of length $100$, split into a training dataset of $300$ sequences and a test dataset of $100$ sequences. Given parameters $W_h, b_h, W_i, W_o, b_o$, the recurrent neural network to predict $y(t)$ is given by:
\begin{equation}
    z_{k+1}(x) = \frac{1}{\sqrt{m}} W_h \sigma(z_k(x)) + b_h + W_i x_k
\end{equation}
\begin{equation}
    \hat{y}_{k}(x) = W_o \sigma(z_k(x)) + b_o
\end{equation}
where $k$ refers to the current recurrent time step, $m$ is the size of the hidden state $z_k(x)$, and $x_k$ and  $\hat{y}_{k}(x)$ refer to the current time input and predicted output respectively. We vary the number of hidden units in the recurrent state in a range from $8$ to $4096$. We train all methods with a learning rate of $0.001$ and batch size $50$ for $200$ epochs on the mean squared error loss.

\paragraph{ImageNet experiments}
We train on CNN architecture with 7 convolutional layers followed by global average pooling and a fully connected layer. All convolutional layers use 3-by-3 convolutions with 512 filters. The first convolutional layer has stride 4-by-4, followed alternating convolutional layers of stride 2-by-2 and 1-by-1. All convolution layers are followed by batch normalization and a ReLU non-linearity. Due to the high memory cost of storing backward projection weights for DFA on ImageNet, we apply the same set of backward projection weights at each spatial location for all layers. Networks are trained on a mean squared error loss with one-hot-encoded labels. We train on ILSVRC-12 images~\cite{russakovsky2015Imagenet} using the following pre-processing steps: we select random crops of the original images with scales in range $[0.08, 1.0]$ and aspect ratio in range $[0.75, 1.34]$; the images are then re-scaled to 224-by-224 pixels. Next, we randomly flip images horizontally and apply a channel-wise normalization. We use a batch size of 100, and train networks using a learning rate of $5$ for 150000 steps.

\subsection{Additional Align variant: Align-prop} \label{app:alignprop}
In this section, we consider an additional variant of Align called Align-prop. Align-prop is equivalent to a variant of Feedback Alignment~\cite{lillicrap2016random} in which the fixed feedback weights are set equal to the feedforward weights at initialization. Although Align-prop is similar to Align-ada, it has the following key difference: Align-ada uses a fixed feedback Jacobian $\nabla_{z_l} f^{(0)}(x)$ to estimate gradients at each layer while Align-prop's approximation of the feedback Jacobian is not fixed. This is because in Align-prop's approximation of the true feedback Jacobian $\nabla_{z_l} f^{(t)}(x)$, although the feedback weights are fixed, the derivatives of activations in intermediate layers change over the course of training. Thus, Align-prop must propagate backwards through each feedback weights at each training step while Align-ada can completely avoid such propagation by storing $\nabla_{z_l} f^{(0)}(x)$.

Furthermore, we note that Align-prop satisfies the conditions required of Align-ada and Align-zero in Proposition 2. This is because Align-prop's approximation of the true feedback Jacobian $\nabla_{z_l} f^{(t)}(x)$ changes infinitesimally from initialization in infinite width networks since activations in infinite width networks change infinitesimally over the course of training. Thus, Align-prop approaches gradient descent in the limit of infinite width networks.

We conduct experiments comparing Align-prop to other baselines on 8 layer CNNs trained on CIFAR-10; we leave a more extensive empirical study of Align-prop as a future work. We experiment on wide networks with width $\times 128$, $\times 256$ and $\times 512$. Otherwise, we use the same architecture and hyperparameter settings used in our main experiments on CIFAR-10 (see Appendix~\ref{app: hyperparams}).

As observed in Table~\ref{tab: scaling_accs_pool_graddescent_alignprop}, Align-prop outperforms all biologically-motivated baselines including other Align variants. We suspect that this is because Align-prop is allowed to adapt its approximation of the layerwise feedback Jacobian over the course of training: the activation patterns in the backward pass of Align-prop match those used during the forward pass.

\begin{table*}[!h]
  \centering
      \caption{Test set accuracies of 8-layer convolutional neural networks with different layer widths on the CIFAR-10 dataset with different methods of training. The scale of the architecture is denoted $\times n$, where $n$ is the number of filters in intermediate layers. For each scale, the best performing biologically-motivated method is in \textbf{bold}.}
    \adjustbox{max width=\textwidth}{
    \begin{tabular}{lrrr}
    \toprule
    Scale & \multicolumn{1}{l}{x128} & \multicolumn{1}{l}{x256} & \multicolumn{1}{l}{x512} \\
    \midrule
    Align-prop & \textbf{59.2}\% & \textbf{59.2}\% & \textbf{59.6}\% \\
    Align-ada  & {54.2}\% & {56.6}\% & {58.0}\% \\
    Align-zero & 40.0\% & 47.6\% & 51.4\% \\
    DFA~\cite{nokland2016direct}    & {54.9}\% & 54.5\% & 54.1\% \\
    FA~\cite{lillicrap2016random}  & 45.7\% & 45.6\% & 45.4\% \\
    Last layer & 34.8\% & 40.2\% & 43.4\% \\
    \textcolor{gray}{Normal} & \textcolor{gray}{64.2\%} & \textcolor{gray}{63.2\%} & \textcolor{gray}{62.3\%} \\
    \bottomrule
    \end{tabular}%
    }

  \label{tab: scaling_accs_pool_graddescent_alignprop}%
\end{table*}

\subsection{Standard parameterization experiments} \label{app: std_param}
In this section, we train networks with Align-ada and Align-zero under standard parameterization. We note that our theoretical analysis and algorithmic contributions are primarily applicable to networks trained in the NTK training regime; nevertheless, we consider to what extent our proposed learning rules can extend to standard parameterization.

Recall that in Equation~\ref{eqn: ntk_param}, the network is defined under NTK parameterization as:
\begin{equation}
    z_{l}(x)  = \frac{1}{\sqrt{m_{l-1}}} W_l \sigma(z_{l-1}(x)) + b_l
\end{equation}
where $m_{l-1}$ is the dimensionality of $z_{l-1}(x)$. Under \textit{standard parameterization}, we omit the width dependent normalization in the definition of the network:
\begin{equation}
    z_{l}(x)  = W_l \sigma(z_{l-1}(x)) + b_l
\end{equation}
However, under standard parameterization, weight initializations are appropriately scaled such that the distribution of networks are the same as under NTK parameterization \textit{at initialization}. During training, weights and biases evolve differently under the two parameterizations. See~\cite{jacot2018NTK} for further discussion of the difference between NTK and standard parameterization.

We train CNNs with 7 convolutional layers under the standard parameterization on CIFAR-10. We follow the training procedure detailed in Appendix~\ref{app: hyperparams} with the following differences: we perform an extensive grid search over learning rates in $\{0.0005, 0.001, 0.002, 0.005, 0.01, 0.02, 0.05, 0.1, 0.2, 0.5 \}$, and we use Adam~\cite{kingma2015adam} to optimize network parameters. To better understand the differences in training behavior between standard parameterization and NTK parameterization, we also consider two additional settings: 1) training only for 1 epoch, 2) training only for 1 epoch with a small learning rate of $0.0005$.

As observed in Appendix~\ref{app: tables} Table~\ref{tab:std_param_cifar}, Normal achieves $>85 \%$ test set accuracy on the widest networks tested. Align-ada incurs a performance gap with Normal of $20 -30 \%$, although its absolute performance in this setting is higher than under the NTK regime (see Appendix~\ref{app: tables} Table~\ref{tab: scaling_accs_pool_graddescent} for tabulated results in the NTK regime). Align-zero is unstable in this training setting. These experiments validate that the Align methods' comparable performance to gradient descent relies on the NTK regime, as expected by our theoretical analysis.

We also observe that as in settings more closely matching the NTK regime, the gap between Align-ada and gradient descent decreases. In particular, we find more comparable performance when limiting training to 1 epoch. This is practically relevant in settings with a limited computational budget available for training: in these settings, Align-ada performs closer to the maximum performance expected (i.e. by training with gradient descent). We also observe that when training with a small learning rate, the test set accuracy gap between Align-ada and Normal decreases, particularly for narrower networks. This is reasonable since in the NTK training regime, parameters move in a small neighborhood around their initialization points~\cite{jacot2018NTK}, which can be qualitatively matched under standard parameterization by restricting training time and learning rate.

\subsection{Seed robustness experiments} \label{app:multiseed}
To verify the statistical significance of our results, we conduct multiple trials of one of our experiments under multiple random seeds. Specifically, we train CNNs with 7 convolutional layers at a width scale of $\times 256$ on CIFAR-10 using 5 new random seeds. The architecture and hyperparameters are chosen to be consistent with our main experiments; in particular, we train for 150 epochs.

As observed in Table~\ref{tab:multiseed}, the trends of the results over multiple random seeds are consistent with what we observe in our main experiments (see Table~\ref{tab: scaling_accs_pool_graddescent}). Moreover, the standard deviations of the test set accuracies are less than $1 \%$ for each method. This gives us confidence in the statistical significance of our results.

\begin{table}[h!]
\centering
\caption{Test set accuracies of 8-layer convolutional neural with $256$ filters per intermediate layer on the CIFAR-10 dataset with different methods of training. Means and standard deviations are reported over $5$ random trials with different random seeds.}
\begin{tabular}{llllll}
\toprule
Method   & Align-ada    & Align-zero  & DFA                            & FA                             & \textcolor{gray}{Normal}                         \\ 
\midrule
Accuracy & $56.4 \pm 0.4 \%$ & $46.9 \pm 0.5 \%$ & $55.3 \pm 0.7 \%$ & $46.1 \pm 0.5 \%$ & \textcolor{gray}{$63.1 \pm 0.4 \%$} \\
\bottomrule
\end{tabular}
\label{tab:multiseed}
\end{table}

\subsection{Training time of Align methods} \label{app:train_time}
In this section, we evaluate the training time of Align methods compared to backpropagation. We highlight that this is done to assess the possibility of using Align methods as a learning rule for practical applications rather than to evaluate the biological plausibility of Align methods. Furthermore, we note that while the implementation of backprop has been extensively optimized due to its popularity, Align methods have not be optimized to run efficiently on standard computing hardware.

We compare training times of Align-ada and Align-zero with backprop. We evaluate on CNNs with 7 convolutional layers at a variety of architectural scales. The architecture is chosen to be consistent with our main experiments. For each training method and architecture scale, we measure training times over 15 trials.

As observed in Figure~\ref{fig:train_time}, we find that Align methods achieve generally comparable training times to backprop, with training times of all methods scaling similarity with network width. Align methods appear to be nearly a fixed factor slower than backprop, with Align-zero faster than Align-ada. This is because our implementation of Align methods stores additional information about the network state at initialization in order to approximate gradients.

\begin{figure}
    \centering
    \includegraphics[width=0.75\textwidth]{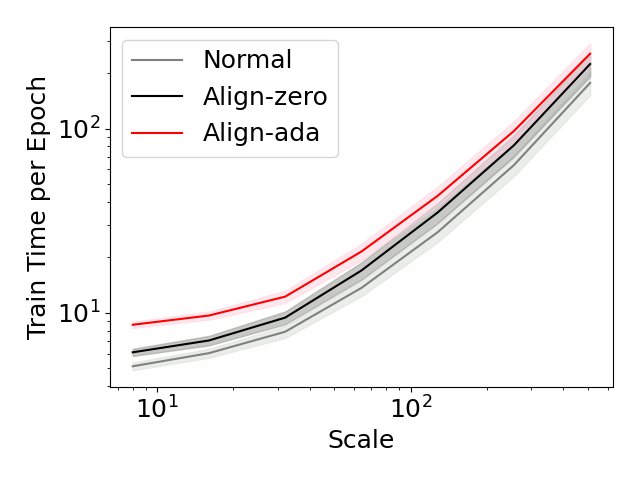}
    \caption{Per epoch training times of different learning rules vs. network scale of CNNs trained on CIFAR-10. Margins indicate standard errors over 15 trials.}
    \label{fig:train_time}
\end{figure}

\subsection{Results on Add task} \label{app: add_task}
\paragraph{Dataset and baselines}
We train standard ReLU-activated RNNs on the Add task used as a benchmark for biologically-plausible RNN training rules in ~\cite{marschall2020unified}. Additional details on the dataset, architectures and hyperparameters are included in Appendix~\ref{app: hyperparams}. We compare Align-ada and Align-zero to Normal and Readout-only, which trains only readout weights and fixes all other weights.
\paragraph{Varying network width}
In Appendix~\ref{app: tables} Figure~\ref{fig:width_tuning_rnn}, we plot the training and test set losses of different learning rules. Although we find a large gap between Normal and alignment-based methods at small width, at larger widths, the gap decreases with Align methods performing comparably to Normal at a network scale of $512$ hidden units. At large widths, the performance of Align methods far exceed Readout-only training, indicating that tuning the hidden layer representation is useful even at large width when random representations provide enough information to adequately perform the task. See Appendix~\ref{app: tables} for full tabulated results.

\clearpage
\subsection{Additional tables and figures}
\label{app: tables}

\begin{figure}[!h]
    \centering
    \begin{subfigure}{.49\textwidth}
      \centering
      \includegraphics[width=\linewidth]{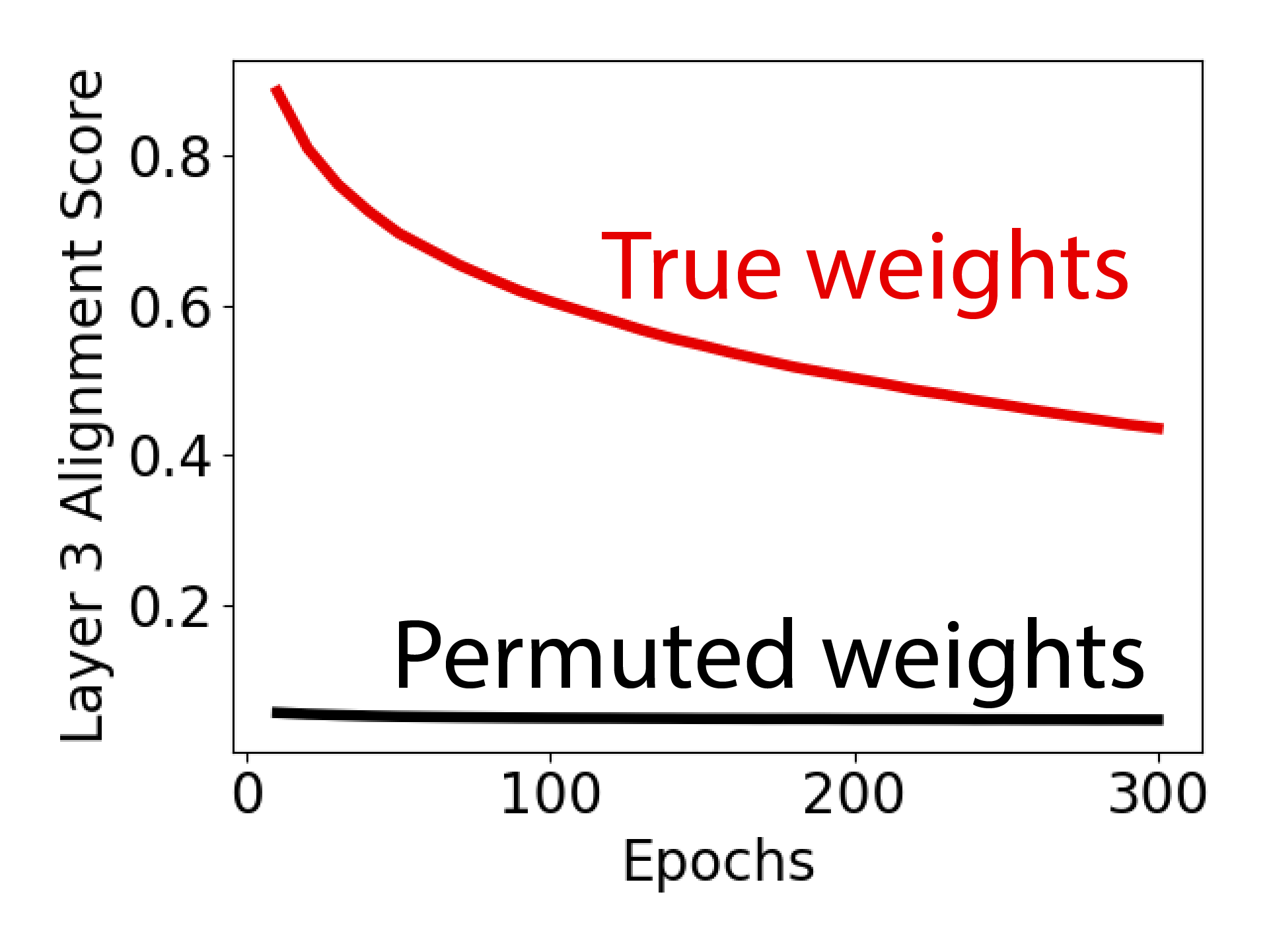}
      \caption{CIFAR-10}
      \label{fig:align_epoch_cifar}
    \end{subfigure}%
    \begin{subfigure}{.49\textwidth}
      \centering
      \includegraphics[width=\linewidth]{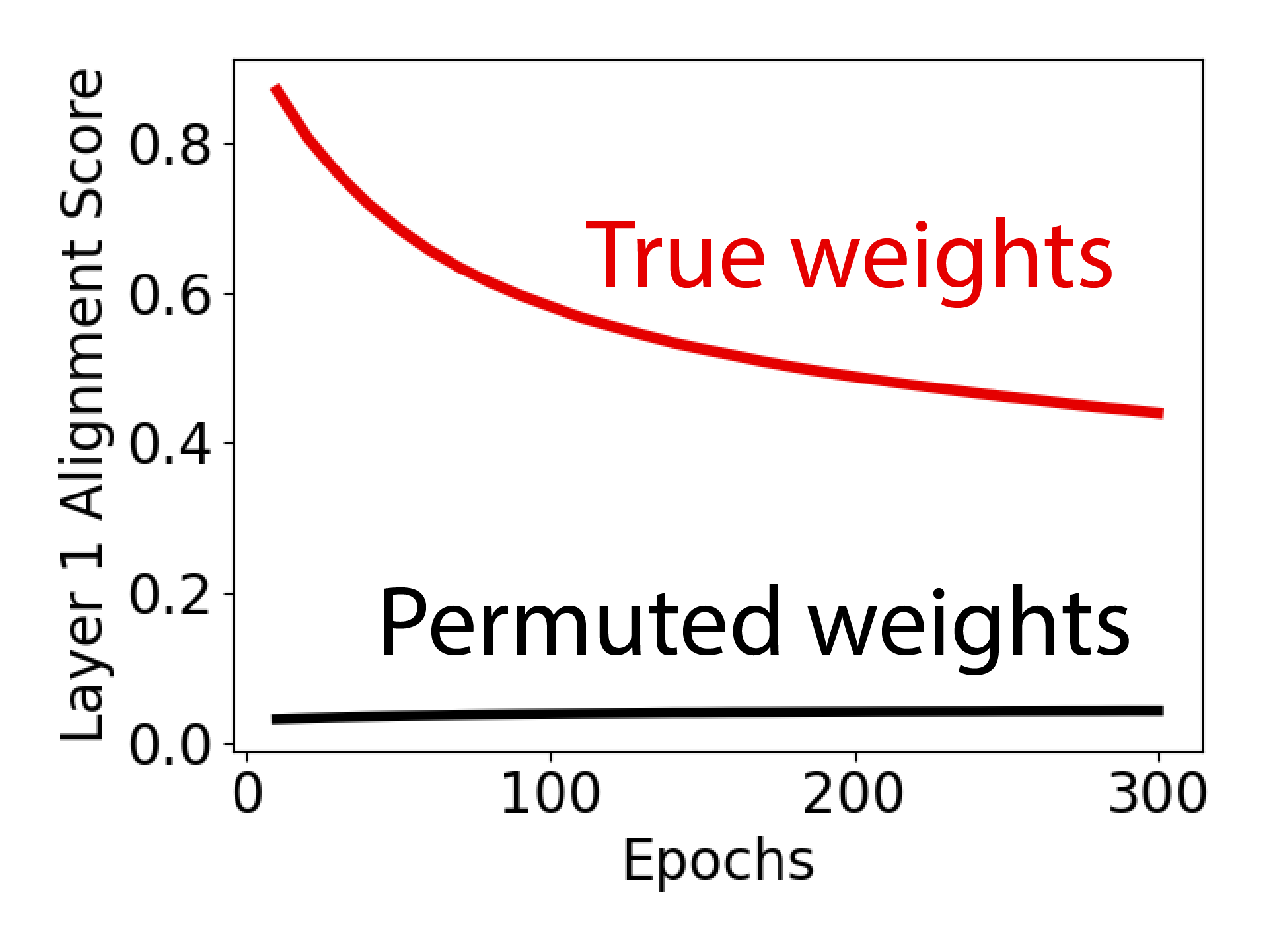}
      \caption{KMNIST}
      \label{fig:align_epoch_kmnist}
    \end{subfigure}
    
    \caption{Alignment scores of convolutional neural networks over the course of training on the CIFAR-10 and KMNIST. As a baseline, alignment scores in networks with randomly permuted weights are also plotted. Training is conducted for 300 epochs to measure the evolution of alignment scores over a longer time-scale.}
    \label{fig:align_epoch}
\end{figure}

\begin{table*}[!h]
  \centering
        \caption{Alignment scores of 8-layer ReLU-activated convolutional neural networks with different layer widths trained on the CIFAR-10 dataset. The network is trained for 150 epochs at a learning rate of 1. The scale of the architecture is denoted $\times n$, where $n$ is the number of filters in intermediate layers. Alignment scores are shown for all intermediate layers of each network after training. Non-monotonicity with layer width is due to statistical fluctuations.}
    \begin{tabular}{lrrrrrrr}
    \toprule
    Alignment & \multicolumn{1}{l}{x8} & \multicolumn{1}{l}{x16} & \multicolumn{1}{l}{x32} & \multicolumn{1}{l}{x64} & \multicolumn{1}{l}{x128} & \multicolumn{1}{l}{x256} & \multicolumn{1}{l}{x512} \\
    \midrule
    Layer 1 & 0.430 & 0.465 & 0.527 & 0.502 & 0.651 & 0.720 & 0.749 \\
    Layer 2 & 0.490 & 0.433 & 0.409 & 0.409 & 0.656 & 0.658 & 0.760 \\
    Layer 3 & 0.501 & 0.467 & 0.446 & 0.515 & 0.586 & 0.669 & 0.766 \\
    Layer 4 & 0.372 & 0.297 & 0.350 & 0.371 & 0.497 & 0.703 & 0.767 \\
    Layer 5 & 0.499 & 0.333 & 0.438 & 0.499 & 0.546 & 0.650 & 0.711 \\
    Layer 6 & 0.301 & 0.363 & 0.329 & 0.321 & 0.419 & 0.500 & 0.603 \\
    Layer 7 & 0.476 & 0.382 & 0.432 & 0.405 & 0.333 & 0.398 & 0.450 \\
    \bottomrule
    \end{tabular}%

  \label{tab: alignment_metrics_cifar}%
\end{table*}%

\begin{table*}[!h]
  \centering
        \caption{Alignment scores of 4-layer convolutional neural networks with different layer widths trained on the KMNIST dataset. The scale of the architecture is denoted $\times n$, where $n$ is the number of filters in intermediate layers. Alignment scores are shown for all intermediate layers of each network after training. Non-monotonicity with layer width is due to statistical fluctuations.}
    \begin{tabular}{lrrrrrrr}
    \toprule
    Alignment & \multicolumn{1}{l}{x8} & \multicolumn{1}{l}{x16} & \multicolumn{1}{l}{x32} & \multicolumn{1}{l}{x64} & \multicolumn{1}{l}{x128} & \multicolumn{1}{l}{x256} & \multicolumn{1}{l}{x512} \\
    \midrule
    Layer 1 & 0.456 & 0.585 & 0.579 & 0.646 & 0.743 & 0.721 & 0.800 \\
    Layer 2 & 0.137 & 0.179 & 0.210 & 0.299 & 0.381 & 0.482 & 0.656 \\
    Layer 3 & 0.165 & 0.207 & 0.232 & 0.253 & 0.376 & 0.471 & 0.581 \\
    \bottomrule
    \end{tabular}%

  \label{tab: alignment_metrics_kmnist}%
\end{table*}%

\begin{table*}[!h]
  \centering
        \caption{Alignment scores of 8-layer Tanh-activated convolutional neural networks with different layer widths trained on the CIFAR-10 dataset. The network is trained for 100 epochs at a learning rate of 2. The scale of the architecture is denoted $\times n$, where $n$ is the number of filters in intermediate layers. Alignment scores are shown for all intermediate layers of each network after training. Non-monotonicity with layer width is due to statistical fluctuations.}
    \begin{tabular}{lrrrrrrr}
    \toprule
    Alignment & \multicolumn{1}{l}{x8} & \multicolumn{1}{l}{x16} & \multicolumn{1}{l}{x32} & \multicolumn{1}{l}{x64} & \multicolumn{1}{l}{x128} & \multicolumn{1}{l}{x256} & \multicolumn{1}{l}{x512} \\
    \midrule
    Layer 1 & 0.459 & 0.605 & 0.681 & 0.816 & 0.829 & 0.925 & 0.935 \\
    Layer 2 & 0.600 & 0.513 & 0.565 & 0.825 & 0.785 & 0.912 & 0.945 \\
    Layer 3 & 0.660 & 0.539 & 0.480 & 0.695 & 0.811 & 0.905 & 0.932 \\
    Layer 4 & 0.470 & 0.324 & 0.504 & 0.741 & 0.788 & 0.887 & 0.919 \\
    Layer 5 & 0.621 & 0.606 & 0.608 & 0.800 & 0.779 & 0.849 & 0.910 \\
    Layer 6 & 0.478 & 0.547 & 0.537 & 0.684 & 0.698 & 0.819 & 0.871 \\
    Layer 7 & 0.439 & 0.675 & 0.530 & 0.587 & 0.595 & 0.759 & 0.827 \\
    \bottomrule
    \end{tabular}%

  \label{tab: alignment_metrics_cifar_tanh}%
\end{table*}%

\begin{table*}[!h]
  \centering
      \caption{Test set accuracies of 8-layer convolutional neural networks with different layer widths on the CIFAR-10 dataset with different methods of training. The scale of the architecture is denoted $\times n$, where $n$ is the number of filters in intermediate layers. For each scale, the best performing biologically-motivated method is in \textbf{bold}.}
    \adjustbox{max width=\textwidth}{
    \begin{tabular}{lrrrrrrr}
    \toprule
    Scale & \multicolumn{1}{l}{x8} & \multicolumn{1}{l}{x16} & \multicolumn{1}{l}{x32} & \multicolumn{1}{l}{x64} & \multicolumn{1}{l}{x128} & \multicolumn{1}{l}{x256} & \multicolumn{1}{l}{x512} \\
    \midrule
    Align-ada & {33.7}\% & {35.7}\% & {42.8}\% & \textbf{49.9}\% & {54.2}\% & \textbf{56.6}\% & \textbf{58.0}\% \\
    Align-zero & 20.9\% & 22.0\% & 30.1\% & 32.6\% & 40.0\% & 47.6\% & 51.4\% \\
    DFA~\cite{nokland2016direct}   & 28.5\% & 34.8\% & 43.0\% & \textbf{49.9}\% & \textbf{54.9}\% & 54.5\% & 54.1\% \\
    FA~\cite{lillicrap2016random}    & \textbf{36.1}\% & \textbf{40.3}\% & \textbf{43.1}\% & 46.5\% & 45.7\% & 45.6\% & 45.4\% \\
    Last layer & 20.7\% & 22.5\% & 28.2\% & 31.1\% & 34.8\% & 40.2\% & 43.4\% \\
    \textcolor{gray}{Normal} & \textcolor{gray}{54.8\%} & \textcolor{gray}{61.0\%} & \textcolor{gray}{64.2\%} & \textcolor{gray}{65.3\%} & \textcolor{gray}{64.2\%} & \textcolor{gray}{63.2\%} & \textcolor{gray}{62.3\%} \\
    \bottomrule
    \end{tabular}%
    }

  \label{tab: scaling_accs_pool_graddescent}%
\end{table*}

\begin{table*}[!h]
  \centering
      \caption{Test set accuracies of 4-layer convolutional neural networks with different layer widths on the KMNIST dataset with different methods of training. The scale of the architecture is denoted $\times n$, where $n$ is the number of filters in intermediate layers. For each scale, the best performing biologically-motivated method is in \textbf{bold}.}
    \adjustbox{max width=\textwidth}{
    \begin{tabular}{lrrrrrrr}
    \toprule
    Scale & \multicolumn{1}{l}{x8} & \multicolumn{1}{l}{x16} & \multicolumn{1}{l}{x32} & \multicolumn{1}{l}{x64} & \multicolumn{1}{l}{x128} & \multicolumn{1}{l}{x256} & \multicolumn{1}{l}{x512} \\
    \midrule
    {Align-ada} & {47.5}\% & {58.8}\% & \textbf{70.7}\% & \textbf{71.2}\% & \textbf{76.1}\% & \textbf{77.8}\% & \textbf{78.4}\% \\
    Align-zero & 24.5\% & 29.9\% & 31.2\% & 47.2\% & 50.3\% & 55.2\% & 61.2\% \\
    DFA~\cite{nokland2016direct}   & 37.6\% & 49.4\% & 61.4\% & 70.7\% & 75.8\% & 77.3\% & 77.9\% \\
    FA~\cite{lillicrap2016random}    & \textbf{52.4\%} & \textbf{66.3\%} & 68.6\% & 69.6\% & 72.1\% & 72.8\% & 63.5\% \\
    Last layer & 25.9\% & 35.6\% & 42.2\% & 51.8\% & 58.7\% & 65.6\% & 68.2\% \\
   \textcolor{gray}{Normal} & \textcolor{gray}{72.2\%} & \textcolor{gray}{83.4\%} & \textcolor{gray}{86.2\%} & \textcolor{gray}{87.2\%} & \textcolor{gray}{87.5\%} & \textcolor{gray}{87.4\%} & \textcolor{gray}{87.2\%} \\
    \bottomrule
    \end{tabular}%
    }

  \label{tab: scaling_accs_pool_graddescent_kmnist}%
\end{table*}

\begin{figure}
    \centering
    \includegraphics[width=0.7\textwidth]{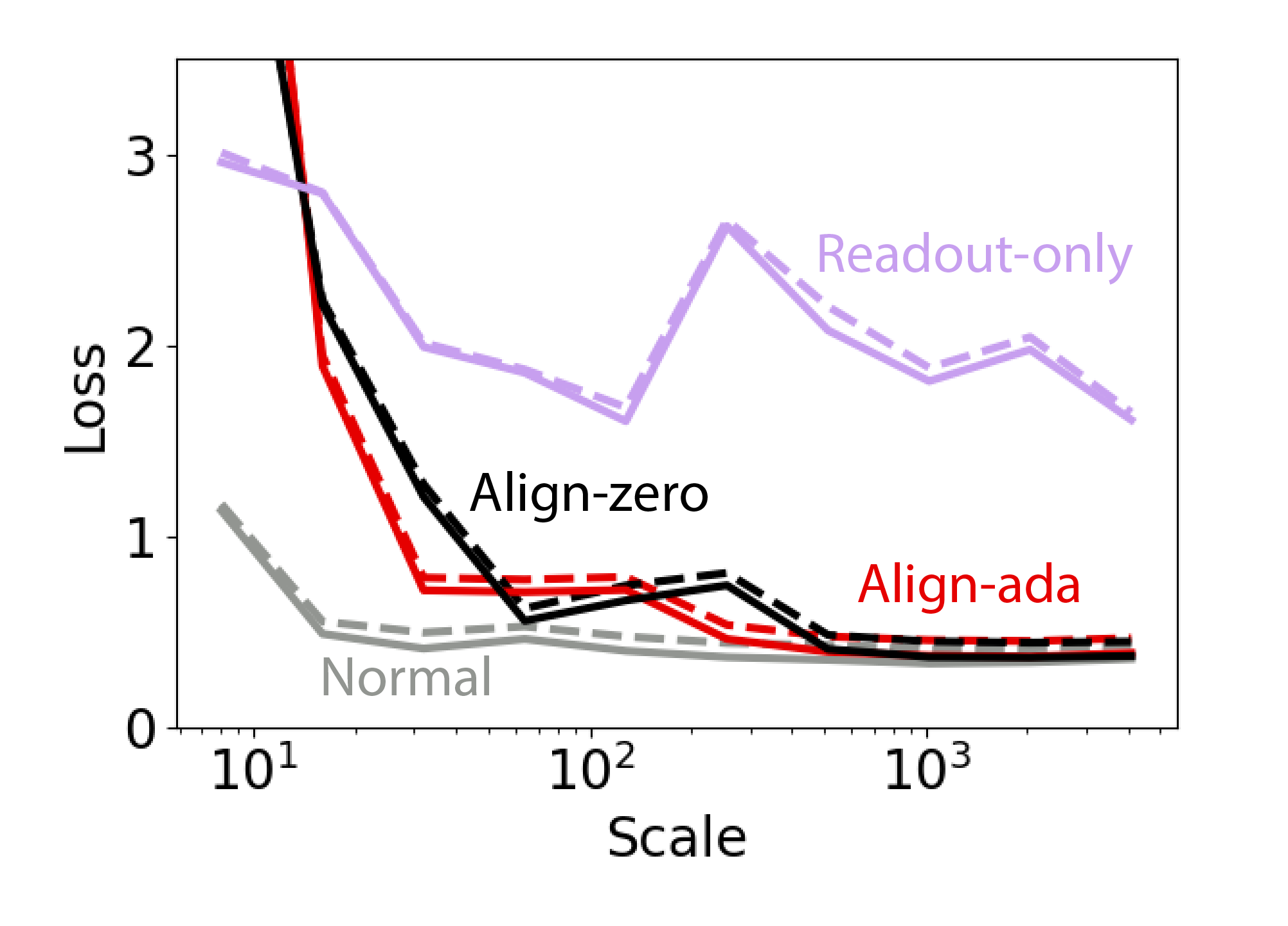}
    \caption{Loss of different learning rules (Normal, Align-ada, Align-zero, Readout-only) as a function of architecture scale on a RNN. Solid lines: test loss; dashed: train loss.}
    \label{fig:width_tuning_rnn}
\end{figure}

\begin{table*}[!h]
  \centering
      \caption{Training and test loss for RNNs trained with normal, Align-ada,  Align-zero and readout-only training on the Add task. Training losses are higher than test set losses due to the test set being an easier task: random chance loss is 9.359 on the training set and 9.312 on the test set.}
     \adjustbox{max width=\textwidth}{
    \begin{tabular}{lrrrrrrrrrr}
    \toprule
    Scale & \multicolumn{1}{l}{x8} & \multicolumn{1}{l}{x16} & \multicolumn{1}{l}{x32} & \multicolumn{1}{l}{x64} & \multicolumn{1}{l}{x128} & \multicolumn{1}{l}{x256} & \multicolumn{1}{l}{x512} & \multicolumn{1}{l}{x1024} & \multicolumn{1}{l}{x2048} & \multicolumn{1}{l}{x4096} \\
    \midrule
    Align-ada, Train loss & 6.652 & 1.944 & 0.782 & 0.772 & 0.786 & 0.535 & 0.474 & 0.455 & 0.450 & 0.466 \\
    Align-ada, Test loss & 6.795 & 1.890 & 0.716 & 0.707 & 0.718 & 0.461 & 0.394 & 0.374 & 0.369 & 0.388 \\
    Align-zero, Train loss & 5.228 & 2.235 & 1.272 & 0.621 & 0.741 & 0.807 & 0.482 & 0.448 & 0.442 & 0.447 \\
    Align-zero, Test loss & 5.227 & 2.211 & 1.205 & 0.556 & 0.663 & 0.742 & 0.408 & 0.367 & 0.364 & 0.372 \\
    Readout-only, Train loss & 3.008 & 2.798 & 2.011 & 1.872 & 1.674 & 2.655 & 2.202 & 1.881 & 2.042 & 1.647 \\
    Readout-only, Test loss & 2.957 & 2.794 & 1.990 & 1.855 & 1.601 & 2.623 & 2.077 & 1.811 & 1.976 & 1.608 \\
    \textcolor{gray}{Normal, Train loss} & \textcolor{gray}{1.168} & \textcolor{gray}{0.553} & \textcolor{gray}{0.495} & \textcolor{gray}{0.527} & \textcolor{gray}{0.474} & \textcolor{gray}{0.440} & \textcolor{gray}{0.429} & \textcolor{gray}{0.408} & \textcolor{gray}{0.412} & \textcolor{gray}{0.429} \\
    \textcolor{gray}{Normal, Test loss} & \textcolor{gray}{1.133} & \textcolor{gray}{0.488} & \textcolor{gray}{0.411} & \textcolor{gray}{0.461} & \textcolor{gray}{0.399} & \textcolor{gray}{0.366} & \textcolor{gray}{0.352} & \textcolor{gray}{0.333} & \textcolor{gray}{0.339} & \textcolor{gray}{0.357} \\
    \bottomrule
    \end{tabular}%
    }

  \label{tab:rnn}%
\end{table*}%

\begin{table*}[!h]
  \centering
      \caption{Test set accuracies of 8-layer convolutional neural networks with different layer widths on the CIFAR-10 dataset with different methods of training (Align-ada, Align-zero, FA~\cite{lillicrap2016random}, DFA~\cite{nokland2016direct}, Normal). Results are reported when trained on different numbers of randomly sampled CIFAR-10 training points. The scale of the architecture is denoted $\times n$, where $n$ is the number of filters in intermediate layers. For each scale and dataset size, the best performing biologically-motivated method is in \textbf{bold}.}
    \adjustbox{max width=\textwidth}{
    \begin{tabular}{rlrrrrrrr}
    \toprule
    \multicolumn{1}{l}{Training data} & Method & \multicolumn{1}{l}{x8} & \multicolumn{1}{l}{x16} & \multicolumn{1}{l}{x32} & \multicolumn{1}{l}{x64} & \multicolumn{1}{l}{x128} & \multicolumn{1}{l}{x256} & \multicolumn{1}{l}{x512} \\
    \midrule
    45    & Align-ada & \textbf{28.9\%} & \textbf{31.1\%} & \textbf{28.9\%} & \textbf{31.1\%} & \textbf{31.1\%} & \textbf{31.1\%} & \textbf{28.9\%} \\
    45    & Align-zero & 24.4\% & 17.8\% & 22.2\% & 24.4\% & 28.9\% & 24.4\% & 20.0\% \\
    45    & FA    & 24.4\% & 24.4\% & 24.4\% & 26.7\% & 26.7\% & 26.7\% & 17.8\% \\
    45    & DFA   & 17.8\% & 28.9\% & \textbf{28.9\%} & 28.9\% & 28.9\% & 28.9\% & 20.0\% \\
    \textcolor{gray}{45}    & \textcolor{gray}{Normal} & \textcolor{gray}{24.4\%} & \textcolor{gray}{17.8\%} & \textcolor{gray}{22.2\%} & \textcolor{gray}{20.0\%} & \textcolor{gray}{22.2\%} & \textcolor{gray}{24.4\%} & \textcolor{gray}{26.7\%} \\
    \midrule
    150   & Align-ada & {20.7\%} & 22.7\% & 24.0\% & \textbf{27.3\%} & \textbf{30.0\%} & \textbf{30.0\%} & 27.3\% \\
    150   & Align-zero & \textbf{21.3\%} & 20.0\% & \textbf{24.7\%} & 23.3\% & 29.3\% & 27.3\% & \textbf{28.0\%} \\
    150   & FA    & 20.0\% & \textbf{23.3\%} & 23.3\% & 25.3\% & 25.3\% & 25.3\% & 22.0\% \\
    150   & DFA   & {20.7\%} & 20.7\% & 21.3\% & 21.3\% & 21.3\% & 21.3\% & 22.7\% \\
    \textcolor{gray}{150}   & \textcolor{gray}{Normal} & \textcolor{gray}{22.0\%} & \textcolor{gray}{24.0\%} & \textcolor{gray}{24.7\%} & \textcolor{gray}{24.7\%} & \textcolor{gray}{29.3\%} & \textcolor{gray}{28.7\%} & \textcolor{gray}{28.7\%} \\
    \midrule
    450   & Align-ada & 25.1\% & 25.3\% & \textbf{28.2\%} & \textbf{30.4\%} & \textbf{33.3\%} & \textbf{33.3\%} & \textbf{35.8\%} \\
    450   & Align-zero & 22.0\% & 22.0\% & 26.4\% & 26.0\% & 29.1\% & 32.0\% & 32.4\% \\
    450   & FA    & 25.6\% & 25.6\% & 25.6\% & 25.6\% & 25.6\% & 25.6\% & 24.7\% \\
    450   & DFA   & \textbf{26.7\%} & \textbf{26.7\%} & 26.7\% & 28.0\% & 28.0\% & 28.0\% & 25.1\% \\
    \textcolor{gray}{450}   & \textcolor{gray}{Normal} & \textcolor{gray}{28.7\%} & \textcolor{gray}{29.6\%} & \textcolor{gray}{32.0\%} & \textcolor{gray}{31.3\%} & \textcolor{gray}{32.7\%} & \textcolor{gray}{33.1\%} & \textcolor{gray}{34.2\%} \\
    \midrule
    1500  & Align-ada & 26.8\% & \textbf{29.9\%} & \textbf{31.3\%} & \textbf{35.5\%} & \textbf{36.1\%} & \textbf{40.2\%} & \textbf{39.1\%} \\
    1500  & Align-zero & 20.7\% & 23.1\% & 28.4\% & 31.2\% & 33.2\% & 37.0\% & 37.3\% \\
    1500  & FA    & \textbf{29.1\%} & 29.1\% & 29.1\% & 29.1\% & 29.1\% & 29.1\% & 25.8\% \\
    1500  & DFA   & 24.9\% & 28.8\% & 30.1\% & 30.1\% & 30.1\% & 30.1\% & 28.6\% \\
    \textcolor{gray}{1500}  & \textcolor{gray}{Normal} & \textcolor{gray}{32.9\%} & \textcolor{gray}{33.1\%} & \textcolor{gray}{36.5\%} & \textcolor{gray}{39.5\%} & \textcolor{gray}{39.3\%} & \textcolor{gray}{40.3\%} & \textcolor{gray}{41.0\%} \\
    \midrule
    4500  & Align-ada & \textbf{30.3\%} & \textbf{33.1\%} & \textbf{34.9\%} & \textbf{40.0\%} & \textbf{42.3\%} & \textbf{45.2\%} & \textbf{46.2\%} \\
    4500  & Align-zero & 22.3\% & 23.1\% & 30.1\% & 31.8\% & 36.4\% & 41.9\% & 42.6\% \\
    4500  & FA    & 28.9\% & 31.6\% & 33.5\% & 35.0\% & 35.0\% & 35.0\% & 33.0\% \\
    4500  & DFA   & 26.4\% & 30.2\% & 34.4\% & 39.0\% & 40.4\% & 40.4\% & 36.0\% \\
    \textcolor{gray}{4500}  & \textcolor{gray}{Normal} & \textcolor{gray}{39.3\%} & \textcolor{gray}{43.8\%} & \textcolor{gray}{45.5\%} & \textcolor{gray}{48.4\%} & \textcolor{gray}{47.9\%} & \textcolor{gray}{49.2\%} & \textcolor{gray}{49.2\%} \\
    \midrule
    15000 & Align-ada & 32.5\% & 35.2\% & 40.6\% & 46.1\% & 49.8\% & \textbf{51.8\%} & \textbf{53.3\%} \\
    15000 & Align-zero & 21.7\% & 23.1\% & 30.4\% & 31.8\% & 38.9\% & 46.2\% & 48.5\% \\
    15000 & FA    & \textbf{33.9\%} & \textbf{38.9\%} & 40.4\% & 44.2\% & 44.2\% & 44.2\% & 41.9\% \\
    15000 & DFA   & 26.8\% & 33.2\% & \textbf{40.7\%} & \textbf{47.1\%} & \textbf{51.0\%} & 51.0\% & 50.0\% \\
    \textcolor{gray}{15000} & \textcolor{gray}{Normal} & \textcolor{gray}{50.1\%} & \textcolor{gray}{54.4\%} & \textcolor{gray}{58.9\%} & \textcolor{gray}{58.1\%} & \textcolor{gray}{57.3\%} & \textcolor{gray}{57.2\%} & \textcolor{gray}{57.4\%} \\
    \midrule
    45000 & Align-ada & 33.7\% & 35.7\% & 42.8\% & \textbf{49.9\%} & 54.2\% & \textbf{56.6\%} & \textbf{58.0\%} \\
    45000 & Align-zero & 20.9\% & 22.0\% & 30.1\% & 32.6\% & 40.0\% & 47.6\% & 51.4\% \\
    45000 & FA    & \textbf{36.1\%} & \textbf{40.3\%} & \textbf{43.1\%} & 46.5\% & 45.7\% & 45.6\% & 45.4\% \\
    45000 & DFA   & 28.5\% & 34.8\% & 43.0\% & \textbf{49.9\%} & \textbf{54.9\%} & 54.5\% & 54.1\% \\
    \textcolor{gray}{45000} & \textcolor{gray}{Normal} & \textcolor{gray}{54.8\%} & \textcolor{gray}{61.0\%} & \textcolor{gray}{64.2\%} & \textcolor{gray}{65.3\%} & \textcolor{gray}{64.2\%} & \textcolor{gray}{63.2\%} & \textcolor{gray}{62.3\%} \\
    \bottomrule
    \end{tabular}%
    }

  \label{tab: scaling_accs_pool_graddescent_lowdata}%
\end{table*}

\begin{table}[htbp]
  \centering
  \caption{Test set accuracies of 8-layer convolutional neural networks with different layer widths on the CIFAR-10 dataset with different methods of training. Networks are trained under standard parameterization using Adam optimization~\cite{kingma2015adam}. The scale of the architecture is denoted $\times n$, where $n$ is the number of filters in intermediate layers. Accuracies are reported when trained with 150 epochs, 1 epoch, and 1 epoch with a small learning rate.}
  \label{tab:std_param_cifar}%
    \begin{tabular}{lrrrrrrr}
    \toprule
    Scale & \multicolumn{1}{l}{x8} & \multicolumn{1}{l}{x16} & \multicolumn{1}{l}{x32} & \multicolumn{1}{l}{x64} & \multicolumn{1}{l}{x128} & \multicolumn{1}{l}{x256} & \multicolumn{1}{l}{x512} \\
    \midrule
    \multicolumn{8}{c}{Standard  parameterization, 150 epochs} \\
    \midrule
    Align-ada & 34.7\% & 41.1\% & 47.4\% & 54.4\% & 58.6\% & 61.7\% & 64.9\% \\
    Align-zero & 17.3\% & 16.2\% & 15.4\% & 13.8\% & 14.0\% & 14.3\% & 12.6\% \\
    \textcolor{gray}{Normal} & \textcolor{gray}{60.9\%} & \textcolor{gray}{73.3\%} & \textcolor{gray}{79.1\%} & \textcolor{gray}{82.8\%} & \textcolor{gray}{84.9\%} & \textcolor{gray}{86.5\%} & \textcolor{gray}{87.6\%} \\
    \midrule
    \multicolumn{8}{c}{Standard parameterization, 1 epoch} \\
    \midrule
    Align-ada & 26.1\% & 27.7\% & 35.4\% & 38.1\% & 40.7\% & 41.4\% & 38.0\% \\
    Align-zero & 14.3\% & 11.0\% & 14.1\% & 12.3\% & 12.1\% & 11.2\% & 12.0\% \\
    \textcolor{gray}{Normal} & \textcolor{gray}{37.3\%} & \textcolor{gray}{43.8\%} & \textcolor{gray}{53.4\%} & \textcolor{gray}{58.7\%} & \textcolor{gray}{59.2\%} & \textcolor{gray}{58.5\%} & \textcolor{gray}{46.2\%} \\
    \midrule
    \multicolumn{8}{c}{Standard parameterization, 1 epoch + small learning rate} \\
    \midrule
    Align-ada & 13.0\% & 24.0\% & 35.4\% & 37.7\% & 40.7\% & 41.4\% & 37.9\% \\
    Align-zero & 8.9\% & 11.0\% & 14.1\% & 10.9\% & 12.1\% & 11.2\% & 11.5\% \\
    \textcolor{gray}{Normal} & \textcolor{gray}{17.9\%} & \textcolor{gray}{32.1\%} & \textcolor{gray}{46.4\%} & \textcolor{gray}{55.2\%} & \textcolor{gray}{57.7\%} & \textcolor{gray}{58.5\%} & \textcolor{gray}{46.2\%} \\
    \bottomrule
    \end{tabular}%
\end{table}%

\end{document}